\documentclass[11pt]{article}
\usepackage{times}
\usepackage{enumitem}
\usepackage{nicefrac}
\usepackage{jmlrutils}
\usepackage{fullpage}

\usepackage{amsmath,amsfonts,amssymb}
\usepackage{xcolor}
\usepackage{algorithm, algorithmic}
\usepackage{xspace}
\usepackage{natbib}
\newtheorem{assumption}{Assumption}

\newcommand{\cA}{{\mathcal A}}
\newcommand{\cX}{\ensuremath{\mathcal{X}}}
\newcommand{\cD}{{\mathcal D}}
\newcommand{\cF}{\ensuremath{\mathcal{F}}}
\newcommand{\cR}{\ensuremath{\mathcal{R}}}
\newcommand{\R}{\ensuremath{\mathbb{R}}}

\newcommand{\cM}{\ensuremath{\mathcal{M}}}

\newcommand{\rE}{{\mathbb E}}

\DeclareMathOperator*{\argmax}{\text{argmax}}
\DeclareMathOperator*{\argmin}{\text{argmin}}

\newcommand{\inner}[2]{\ensuremath{\left\langle #1, #2 \right\rangle}}

\newcommand{\order}{\mathcal{O}}

\newcommand{\BE}{{\mathcal E}}
\newcommand{\trace}{{\mathrm{trace}}}

\newcommand{\cT}{\ensuremath{\mathcal{T}}}
\newcommand{\cZ}{\ensuremath{\mathcal{Z}}}

\newcommand{\pp}{{\tilde{p}}}

\newcommand{\hatd}{{\hat{\delta}}}
\newcommand{\hatc}{{\hat{Z}}}

\newcommand{\br}{{\mathrm{br}}}
\newcommand{\dc}{{\mathrm{dc}}}
\newcommand{\regret}{{\mathrm{Regret}}}

\newcommand{\Deltah}{\ensuremath{\hat{\Delta}}}

\usepackage{hyperref}

\newcommand{\alg}{TS$^3$\xspace}
\newcommand{\olive}{OLIVE\xspace}
\newcommand{\alglong}{Two timeScale Two Sample Thompson Sampling\xspace}
\newcommand{\algdesign}{TS$^2$-D\xspace}
\newcommand{\algdesignlong}{Two timeScale Thompson Sampling with Design\xspace}
\newcommand{\idxset}{\ensuremath{\mathcal{I}}}
\newcommand{\otheridx}{\ensuremath{\mathcal{J}}}
\usepackage{eqparbox}

\newcommand{\mypar}[1]{\paragraph{#1}}

\title{Non-Linear Reinforcement Learning in Large Action Spaces:\\Structural Conditions and Sample-efficiency of Posterior Sampling}
\author{Alekh Agarwal \\alekhagarwal@google.com\\Google Research \and Tong Zhang \\tozhang@google.com\\Google Research \& HKUST}

\begin{document}
\date{}
\maketitle

\begin{abstract}
    Provably sample-efficient Reinforcement Learning (RL) with rich observations and function approximation has witnessed tremendous recent progress, particularly when the underlying function approximators are linear. In this linear regime, computationally and statistically efficient methods exist where the potentially infinite state and action spaces can be captured through a known feature embedding, with the sample complexity scaling with the (intrinsic) dimension of these features. When the action space is finite, significantly more sophisticated results allow non-linear function approximation under appropriate structural constraints on the underlying RL problem, permitting for instance, the learning of good features instead of assuming access to them. In this work, we present the first result for non-linear function approximation which holds for general action spaces under a \emph{linear embeddability} condition, which generalizes all linear and finite action settings. We design a novel optimistic posterior sampling strategy, \alg for such problems, and show worst case sample complexity guarantees that scale with a rank parameter of the RL problem, the linear embedding dimension introduced in this work and standard measures of the function class complexity. 
\end{abstract}


\section{Introduction}
\label{sec:intro}
Designing sample-efficient techniques for Reinforcement Learning (RL) settings with large state and action spaces is a key question at the forefront of RL research. A typical approach relies on the use of function approximation to generalize across the state and/or action spaces. When a sufficiently expressive feature embedding is available to the learner, and linear functions of these features are used for learning, a number of recent results provide computationally and statistically efficient techniques to handle continuous state, as well as action spaces with dependence only on an intrinsic dimensionality of the features~\citep{yang2020reinforcement,jin2020provably,agarwal2020pc}. However, practitioners typically rely on neural networks to parameterize the learning agent, a case not covered by the linear results. A parallel line of work~\citep{jiang2017contextual,sun2019model,du2021bilinear} studies more general settings that allow the use of non-linear function approximation, and provides sample complexity guarantees in terms of a structural parameter called the \emph{Bellman rank} of the RL problem. The power of these conditions is elucidated in recent representation learning results~\citep{agarwal2020flambe,uehara2021finite,modi2021model}, that leverage low Bellman rank to learn a good feature map that captures a near-optimal policy and/or the transition dynamics. However, these methods crucially rely on having a finite number of actions, and the guarantees scale with the cardinality of the action set. With this context, our paper asks the following question:\\
\begin{center}
    \emph{Can we design sample-efficient and non-linear RL approaches for large state and action spaces?}
\end{center}

In this paper, we study this question in the framework of a Markov Decision Process (MDP), and focus on problems satisfying a generalization of the Bellman rank parameter introduced by~\citet{jiang2017contextual}. While the definition of Bellman rank applies to both discrete and continuous action spaces, the \olive algorithm of~\citet{jiang2017contextual} applies only to discrete action spaces. The presence of a small action set facilitates uniform exploration for one time step, which lets the agent collect valuable exploration data in the vicinity of states it has already explored, allowing the discovery of successively better states. When good features are available, like in a linear MDP, a basis in the feature space serves an analogous role and indeed recent works~\citep{foster2021statistical} show that experimental design in the right feature space can replace uniform exploration over discrete actions. 

However, this strategy fails beyond (generalized) linear settings, where there is no easy mechanism for obtaining a good exploration basis over the action set apriori. Indeed, the results of~\citet{hao2021online} imply that avoiding some dependence on the size of the action space in general is unavoidable (see Appendix~\ref{sec:sparse-rl} for further discussion). This situation motivates the investigation of additional structures between the hopeless worst-case result and the limiting small action settings. To this end, our approach is motivated by the recent work of~\citet{Zhang2021FeelGoodTS} on the Feel-Good Thompson Sampling (FGTS) algorithm, which they analyze for bandits and a class of RL problems under a linear embeddability assumption using a modified Thompson Sampling approach. 
\begin{table}
\centering
\begin{tabular}{|>{\centering\arraybackslash}m{5cm}|>{\centering\arraybackslash}m{3cm}|>{\centering\arraybackslash}m{5cm}|}
\hline & Structural complexity for \alg & Other approaches\\ \hline
$V$-type Bellman rank $d_1$ \hspace{2cm}+ $K$ actions & $d_1^2 K$  & \olive~\citep{jiang2017contextual}, $V$-type GOLF~\citep{jin2021bellman}\\ \hline
Linear MDP \hspace{2cm}(effective feature dim $d$) & $d^3$ & LSVI-UCB~\citep{jin2020provably}, $Q$-type GOLF~\citep{jin2021bellman}\\ \hline
Mixture of MDPs w/ Bellman rank $d_1$ + $K$ actions & $d_1^2 K$ & Contextual PC-IGW \hspace{1cm} \citep{foster2021statistical}\\\hline
Rank $d$ dynamics + ranking $L$ out of $K$ items (Example~\ref{ex:comb}) & $d^2 KL$ & ??\\ \hline
\end{tabular}
\label{tbl:settings}
\caption{\small Settings captured by our assumptions and prior approaches for them. Our setting subsumes finite action MDPs with a small $V$-type Bellman rank and infinite action linear MDPs. All problems with a small $Q$-type Bellman rank are not covered (see Appendix~\ref{sec:related}). Prior works use different methods for $V$-type Bellman rank with finite actions and linear MDP with infinite actions, unlike this work.}
\end{table}

\mypar{Our Contributions.} With this context, our work makes the following contributions.

\begin{enumerate}[leftmargin=*]
    \item We introduce a new structural model for RL where a generalized form of Bellman rank is small, and further a linear embeddability assumption on the Bellman error is satisfied. We show that this setting subsumes prior works on both linear MDPs as well as finite action problems with a small Bellman rank. Crucially, in both linear MDPs and finite action problems, the embedding features in our definition are known apriori, while we also handle problems with an \emph{unknown linear embedding}, which constitutes a significant generalization of the prior works. As an example, this allows us to generalize the combinatorial bandits setting to long horizon RL (Table~\ref{tbl:settings}).
    \item We introduce a new algorithm, \alglong(\alg), which is motivated from FGTS~\citep{Zhang2021FeelGoodTS}. The algorithm design incorporates a careful two timescale strategy to solve an online minimax problem for estimating Bellman residuals, and a decoupling between roll-in and roll-out policies crucial for the online minimization of these residuals. Both ideas appear novel relative to prior posterior sampling approaches in the literature. The use of nested posterior sampling to solve online minimax problems might be of independent interest.
    \item We show that \alg solves all problems where the generalized Bellman rank and linear embeddability conditions hold, under additional completeness and realizability assumptions. The guarantees scale polynomially in $1/\epsilon$, horizon $H$, Bellman rank and linear embedding parameters, as well as the function class complexity. The result generalizes both guarantees for linear MDPs in terms of feature dimension with no dependence on action cardinality, as well as Bellman rank-based and representation learning guarantees to linearly embeddable infinite action spaces. Overall, our method provides the first approach to representation learning with continuous action spaces. We summarize the settings covered in this paper in Table~\ref{tbl:settings}.
    \item The sample complexity of \alg scales as $1/\epsilon^4$, which is slower than the optimal $1/\epsilon^2$ scaling. We develop an experimental design based adaptation of \alg which does obtain the desired $1/\epsilon^2$ bound. However, in contrast with the general result of \alg, the \algdesign algorithm needs knowledge of the features in our linear embedding assumption, in order to carry out experimental design. 
\end{enumerate}
We remark that the convergence rate guarantees obtained here likely have room for improvement as they do not match the optimal rates for tabular and linear MDPs. We leave this as an important direction for future work, and comment on the potential sources of looseness in our analysis in the later sections of the paper. We now formalize the setting before discussing our structural assumptions and our approach. Detailed discussion of the related works can be found in Appendix~\ref{sec:related}.

\section{Setting}

We study RL in an episodic, finite horizon Contextual Markov Decision Process (MDP) that is parameterized as $(\cX, \cA, \cR, \cD, P)$, where $\cX$ is a state space, $\cA$ is an action space, $\cR$ is a distribution over rewards, $\cD$ and $P$ are distributions over the initial context and subsequent transitions respectively. An agent observes a context $x^1\sim \cD$ for some fixed distribution $\cD$.\footnote{We intentionally call $x^1$ a context and not an initial state of the MDP as we will soon make certain structural assumptions which depend on the context, but take expectation over the states.} At each time step $h\in\{1,\ldots,H\}$, the agent observes the state $x^h$, chooses an action $a^h$, observes $r^h\sim \cR(\cdot\mid x_h, a_h)$ and transitions to $x^{h+1}\sim P(\cdot \mid x^h, a^h)$. We assume that $x^h$ for any $h > 1$ always includes the context $x^1$ to allow arbitrary dependence of the dynamics and rewards on $x^1$. Following prior work~\citep[e.g.][]{jiang2017contextual}, we assume that $r^h\in[0,1]$ and $\sum_{h=1}^H r^h \in [0,1]$ to capture sparse-reward settings~\citep{jiang2018open}. We make no assumption on the cardinality of the state and/or action spaces, allowing both to be potentially infinite. We use $\pi$ to denote the agent's decision policy, which maps from $\cX\to\Delta(\cA)$, where $\Delta(\cdot)$ represents the space of all probability distributions over a set. The goal of learning is to discover an optimal policy $\pi_\star$, which is always deterministic and conveniently defined in terms of the $Q_\star$ function~\citep[see e.g.][]{puterman2014markov,bertsekas1996neuro}
\begin{equation}
    \pi^h_\star(x^h) = \argmax_{a\in\cA} Q^h_\star(x^h,a),~~Q^h_\star(x^h,a^h) = \underbrace{\rE[r^h + \max_{a'\in\cA}Q_\star^{h+1}(x^{h+1},a')\mid x^h, a^h]}_{\cT^h Q_\star^h(x^h,a^h)},\vspace{-0.4cm}
    \label{eq:qstar}
\end{equation}
where we define $Q_\star^{H+1}(x,a) = 0$ for all $x,a$. 

In this work, we focus on value-based approaches, where the learner has access to a function class $\cF\subseteq\{\cX\times\cA\to [0,1]\}$, and uses this function class to approximate the optimal value function $Q_\star$. Let us denote $[H] := \{1,2,\ldots,H\}$. We make two common expressivity assumptions on $\cF$.

\begin{assumption}[Realizability]
    For all $h\in[H]$, we have $Q^h_\star\in\cF$.
\label{ass:realizable}
\end{assumption}

There are well-known lower bounds on the sample complexity of RL~\citep{sekhari2021agnostic}, even under the structural assumptions we will make when realizability is not approximately satisfied. Let us use $\pi_f(x) = \argmax_{a\in\cA} f(x,a)$ and $f(x) = \max_a f(x,a)$ for any $f\in\cF$.

\begin{assumption}[Completeness]
    For any function $f\in\cF$ and $h\in[H]$, $\exists~g\in\cF$ such that $g(x^h,a^h) = \cT^hf(x^h,a^h) := \rE[r^h + f(x^{h+1},\pi_f(x^{h+1})) \mid x^h,a^h]$.
\label{ass:complete}
\end{assumption}

The completeness assumption is not essential information-theoretically and \olive~\citep{jiang2017contextual} does not require it. However, we make this assumption  to facilitate scaling to infinite action sets, which seems challenging without completeness, as we discuss later. We now present the main structural conditions we impose on the environment, and motivate them with some examples.

\section{Structural Conditions}
\label{sec:structure}

In this section, we present the two main structural conditions, namely generalized Bellman rank and linear embeddability. We define the Bellman residual of $f\in\cF$ using another $g\in\cF$ as:
\begin{equation}
    \BE^h(g,f,x^h,a^h) = g(x^h, a^h) - \cT^h f(x^h, a^h).
    \label{eq:berror}
\end{equation}

We now state a generalized Bellman error decomposition for contextual setups.
\begin{assumption}[Generalized Bellman decomposition]
    We assume that for all $f, f' \in \cF$, and $h\in[H]$, there exist (unknown) functions $u^h, \psi^{h-1}$ and an inner product $\langle \cdot,\cdot\rangle$ such that:
\[
\rE_{x^{h} \sim \pi_{f'}|x^1} \BE^h(f,f;x^{h},\pi_f(x^h))
=  \langle\psi^{h-1}(f',x^1), u^h(f,x^1)\rangle.  
\]
We assume that $\sup_{f\in\cF, x^1\in\cX}\|u^h(f,x^1)\|_2 \leq B_1$.
    \label{ass:bellman}
\end{assumption}

Notice that both factors depend on $x^1$, and a similar definition is recently considered in~\citet{foster2021statistical} to cover contextual RL setups~\citep{hallak2015contextual,abbasi2014online,modi2018markov}. The technical treatment of contextual dependency requires a conversion of RL problems into online learning, as in ~\citet{foster2021statistical} and in this paper. Techniques used by other earlier works on decomposition of average Bellman error cannot handle such a dependency.
This definition allows mixtures of MDPs with a small Bellman rank for each context $x^1$, without any explicit scaling with the number of contexts in our results. We now make an effective dimension assumption on $\phi^h, u^h$ instead of requiring them to be finite dimensional like~\citet{jiang2017contextual}.

\begin{definition}[Effective Bellman rank]
    Given any probability measure $p$ on $\cF$. Let
    \begin{align*}
    \Sigma^{h}(p,x^1) = \rE_{f' \sim p}   \psi^{h}(f',x^1) \otimes\psi^{h}(f',x^1),\quad 
    K^{h}(\lambda)= \sup_{p,x^1}\trace\left( (\Sigma^{h}(p,x^1)+ \lambda I)^{-1} \Sigma^{h}(p,x^1) \right) .
    \end{align*}
    For any $\epsilon>0$, define the effective Bellman rank of the MDP as
    \[
    \br(\epsilon) = \sum_{h=1}^{H} \br^{h-1}(\epsilon),
    \quad \br^{h-1}(\epsilon)=\inf_{\lambda>0} \left\{ K^{h-1}(\lambda): \lambda K^{h-1}(\lambda)
    \leq \epsilon^2 \right\} .
    \]
    \label{def:br}
\end{definition}

When $\text{dim}(\psi^{h-1}) = \text{dim}(u^h) = d$, then $K^h(\lambda) \leq d$ for all $\lambda > 0$ so that $\br(0) \leq dH$. This might appear a factor of $H$ larger than the Bellman rank, but note that we aggregate over the horizon, while taking a supremum like~\citet{jiang2017contextual} simply incurs that factor in the sample complexity analysis instead. More generally, the following result gives the behavior of $\br(\epsilon)$ under typical spectral decay assumptions on $\Sigma^h$. 

\begin{proposition}[Rank bounds under spectral decay]
    Suppose that $\|\psi(f,x^1)\| \leq 1$ for all $f\in\cF$, $x^1\in\cX$. Let $\lambda_i(A)$ denote the $i_{th}$ largest eigenvalue of a psd matrix $A$. Suppose we have:
    \begin{itemize}[leftmargin=*,itemsep=0pt]
        \item \emph{Geometric decay:} $\displaystyle \sup_{h, p, x^1}\lambda_i(\Sigma^h(p,x^1)) \leq \alpha^i$, for $\alpha \in (0,1)$. Then $\br(\epsilon) \leq 2H\cdot \frac{\log(\nicefrac{2}{(\epsilon^2(1-\alpha))})}{\log(\nicefrac{1}{\alpha})}$.
        \item \emph{Polynomial decay:} $\displaystyle \sup_{h, p, x^1} \trace((\Sigma^h(p,x^1))^q) \leq R_q$, for $q\in(0,1)$. Then $\br(\epsilon) \leq H\cdot\left(\frac{R_q}{\epsilon^{2}}\right)^{\nicefrac{q}{1-q}}$.
    \end{itemize}
    \label{prop:spectral}
\end{proposition}
We provide a proof of Proposition~\ref{prop:spectral} in Appendix~\ref{sec:spectral}. \citet{jin2021bellman} study a related log-determinant based effective dimension for Bellman rank, but this definition is more natural in our analysis.

While Assumption~\ref{ass:bellman} and Definition~\ref{def:br} facilitate scaling to infinite dimensional factorizations with a low intrinsic dimension, this is still not sufficient to succeed in problems with large action spaces due to the existing lower bounds on learning under sparse transitions (see \citet{hao2021online} and Appendix~\ref{sec:sparse-rl}). We now introduce the second structural condition, inspired by the recent work of~\citet{Zhang2021FeelGoodTS} on contextual bandits and RL with deterministic dynamics, to handle this issue.

\begin{assumption}[Linearly embeddable Bellman error]
For all $f\in\cF$, $h\in[H]$, $x^h\in\cX$ and $a^h\in\cA$, there exist (unknown) functions $w^h$, $\phi^h$ and an inner product $\langle\cdot,\cdot\rangle$ such that
\[
\BE^h(f,f;x^h,a^h) = \langle w^h(f,x^h),\phi^h(x^h,a^h)\rangle.
\]
We assume that $\sup_{f\in\cF,x^h\in\cX} w^h(f,x^h) \leq B_2$.
\label{assumption:embedding}
\end{assumption}
Note that the function $w^h$ is allowed to depend on $x^h$, which makes this assumption significantly weaker than embedding Bellman errors in a fixed feature space. For instance, suppose $|\cA| = K$. Then we can always define $\phi(x^h,a^h) = e_{a^h}\in\R^K$ to be the indicator of the action $a^h$ and $w^h(f,x^h) = (f(x^h,a_1),\ldots,f(x^h,a_K))$ to satisfy Assumption~\ref{assumption:embedding}. This shows that all finite action problems satisfy this assumption, and so do linear MDPs where $w^h(f,x^h)$ only depends on $f$ and $\phi^h(x^h,a^h)$ are the features defining the transition dynamics~\citep{jin2020provably}. Similar to Assumption~\ref{ass:bellman}, we allow the embeddings $w^h, \phi^h$ to have a small effective dimension. 
\begin{definition}[Effective Embedding Dimension]
Given any probability measure $p$ on $\cF$, define
\begin{align*}
\tilde{\Sigma}^h(p,x) = \rE_{f\sim p}\rE_{a\sim \pi_f(x)} \phi^h(x,a)\otimes\phi^h(x,a),~~\tilde{K}^h(\lambda)= \sup_{p,x} \trace\left( (\tilde{\Sigma}^{h}(p,x)+ \lambda I)^{-1} \tilde{\Sigma}^{h}(p,x) \right).
\end{align*}
For any $\epsilon>0$, define the effective embedding dimension of as
\[
  \dc(\epsilon) = \sum_{h=1}^H\dc^h(\epsilon)  ,  \quad\mbox{where}\quad
  \dc^h(\epsilon) = \inf_{\lambda>0} \left\{ \tilde{K}^{h}(\lambda): \lambda \tilde{K}^{h}(\lambda)\leq \epsilon^2 \right\} .
\]
\label{def:dc}
\end{definition}
The definition of $\tilde{K}^h(\lambda)$ measures the (approximate) condition number of the covariance for the worst distribution $p$. If $\text{dim}(\phi^h)=d$, then $\tilde{K}^h(\lambda) \leq d$. In this case, we get $\dc(0) \leq dH$. For infinite dimensional $\phi^h$, we obtain bounds similar to Proposition~\ref{prop:spectral}.

\mypar{Comparison with Bellman-Eluder dimension~\citep{jin2021bellman}.} \citet{jin2021bellman} present an alternative $Q$-type Bellman rank and its distributional generalization based on the Eluder dimension~\citep{russo2013eluder}. However, both Eluder dimension based results~\citep{wang2020reinforcement,feng2021provably} and the $Q$-type Bellman-Eluder dimension do not capture all finite action, non-linear contextual bandit problems with a realizable reward (see Appendix~\ref{sec:Qtype-lb}). In fact, the $Q$-type Bellman rank has a more linear like structure, and is easier to deal with, because one only requires a single posterior sampling \citep{DMZZ2021-neurips}. The $V$-type Bellman rank, which naturally captures non-linear setups like representation learning, requires a more sophisticated handling via double sampling from the posterior. 
Our setup subsumes the $V$-type Bellman rank of~\citet{jiang2017contextual}, but does not include all problems with a small $Q$-type Bellman rank. The most prominent example of linear MDP with infinite actions covered by the GOLF algorithm of~\citet{jin2021bellman} is also captured by our assumptions. However, the setting of low inherent Bellman error using linear function approximation~\citep{zanette2020learning} is not covered by our assumptions.

So far, we have explained how both the settings of a bounded Bellman rank with a finite action set, as well as linear MDPs with infinite actions satisfy both of our assumptions with good bounds on $\br(\epsilon)$ and $\dc(\epsilon)$. Next we discuss some examples beyond these where our assumptions hold and which go beyond either of these two known special cases. For the examples, we assume that the underlying MDP has low-rank transition dynamics~\citep{jiang2017contextual,jin2020provably} so that $\cT f(x,a) = \langle w_f,\psi(x,a)\rangle$ for all functions $f~:~\cX\times\cA\to[0,1]$ with $\psi(x,a)$ being the state-action features which factorize the dynamics. For such problems, Assumption~\ref{ass:bellman} always holds \citep{jiang2017contextual} and Assumption~\ref{assumption:embedding} is equivalent to checking that $f(x,a)$ is linearly embeddable.

\begin{example}[Linear embedding of combinatorial actions]
    Many recommendation settings consist of combinatorial action spaces such as lists, rankings and page layouts. While these problems are intractable in the worst-case, a line of work originating in combinatorial bandits~\citep{cesa2012combinatorial} assumes linear decomposition of rewards for tractability. For concreteness, let us consider a ranking scenario where the system observes some state features $x$ depending on the current user state and any other side information, and wants to choose a ranking $\mathfrak{a}$ of $L$ items $\alpha_1,\ldots,\alpha_L$ from a base set $\Omega$, with $|\Omega| = K$. We observe that $|\cA| = {K\choose L}\cdot L!$ in this example, which grows as $\order(K^L)$. Hence the sample complexity of direct exploration over rankings is $\order(K^L)$.
    
    Mirroring the setup from combinatorial bandits and its contextual generalization~\citep{swaminathan2017off},~\citet{ie2019reinforcement} propose the SlateQ model where the $Q_\star$ function takes the form:
    \begin{equation}
        Q_\star(x,a) = \sum_{\alpha \in \mathfrak{a}} P(\alpha | x,\mathfrak{a}) g(x,\alpha).
        \label{eq:slateq}
    \end{equation}
    Here $P(\alpha | x,\mathfrak{a})$ is the probability of a user (with features $x$) choosing the item $\alpha$ when the ranking $\mathfrak{a}$ is presented, and $g(x,\alpha)$ is an unknown value of recommending that item in state $x$ to the user. This assumption is motivated from typical approaches in the user-modeling literature in recommendation systems, and effectively posits that the $Q$-value of a ranking $\mathfrak{a}$ in a state depends on the \emph{unknown} values of the base items in that ranking, weighted by a user interaction model $P$.
    
    The user model, which encodes the likelihood of the user choosing a particular item in a ranking is often estimated separately from the RL task using a click probability model, or a cascade model~\citep[see e.g.][for a discussion]{ie2019reinforcement}. In such scenarios, we can define a class $\cF$ by parameterizing the function $g$, and obtain linear embedding of Bellman residuals whenever the MDP also has low-rank dynamics in some features $\psi$, using $\phi(x,\mathfrak{a}) = (P(\cdot\mid x,\mathfrak{a}),\psi(x,\mathfrak{a}))$ as the concatenation of the user model with dynamics features. Here $\psi(x,\mathfrak{a})$ intuitively encodes the information needed to describe how the state $x$ of a user evolves across interactions, which could be low-dimensional, for instance, if there is a set of representative user types.
    
    Notably, these choices lead to a linear embedding dimension which grows as $K$. In many applications, the position of an item in the ranking is a dominant effect, in which case we can consider $\alpha$ to consist of an item-position tuple, leading to an $KL$-dimensional factorization. In either case, this yields an exponential saving in the sample complexity compared with direct exploration in the ranking space, mirroring prior results~\citep{cesa2012combinatorial,swaminathan2017off,ie2019reinforcement}, which either do not consider exploration or work in simpler bandit settings. We are not aware of other existing approaches that can handle rich observation RL and combinatorial action spaces simultaneously. 
    \label{ex:comb}
\end{example}

In Appendix~\ref{sec:examples}, we give two further examples of using a basis expansion in the action space, as well as an infinite action version of representation learning for Block MDPs~\citep{du2019provably}, where Assumption~\ref{assumption:embedding} holds in a natural manner. We now proceed to describe our algorithm.

\section{The \alglong algorithm}
\label{sec:alg}

Having presented our structural conditions, we now present our main algorithm in this section, which is based on Thompson Sampling~\citep{thompson1933likelihood} and its FGTS adaptation in~\citet{Zhang2021FeelGoodTS}. To define the algorithm, we need some additional notation. At any round $t$ of the algorithm, using the observed tuple $(x_t^h, a_t^h, r_t^h, x_t^{h+1})$, we define
\begin{equation}
    \Deltah_t^h(g,f) = g(x_t^h, a_t^h) - r_t^h - f(x_t^{h+1}),
    \label{eq:deltah}
\end{equation}
as a TD approximation of $\BE(g,f,x_t^h, a_t^h)$. The algorithm is presented in Algorithm~\ref{alg:online_TS}.

\begin{algorithm}[tb]
	\caption{\alglong (\alg)}
	\label{alg:online_TS}
		\begin{algorithmic}[1]
		    \REQUIRE Function class $\cF$, prior $p_0\in\Delta(\cF)$, learning rates $\eta, \gamma$ and optimism coefficient $\lambda$.
		    \STATE Set $S_0 = \emptyset$.
			\FOR{$t=1,\ldots,T$}
			    \STATE Observe $x_t^1 \sim \cD$ and draw $h_t \sim \{1,\ldots,H\}$ uniformly at random.
			    \STATE Define $q_t(g) = p(g | f,S_{t-1}) \propto p_0(g)\exp(-\gamma\sum_{s=1}^{t-1} \Deltah_s^{h_s}(g,f)^2)$.\COMMENT{Inner loop TS update}\label{line:inner-update}
			    \STATE Define $L_t^h(f) = \eta\Deltah_t^h(f,f)^2 + \frac{\eta}{\gamma}\ln\rE_{g\sim q_t}\left[\exp(-\gamma\Deltah_t^h(g,f)^2)\right]$.\COMMENT{Likelihood function}\label{line:likelihood}
			    \STATE Define $p_t(f) = p(f|S_{t-1}) \propto p_0(f)\exp(\sum_{s=1}^{t-1}(\lambda f(x_s^1) - L_s^{h_s}(f))$ as the posterior. \COMMENT{Outer loop Optimistic TS update}\label{line:outer-update}
			    \STATE Draw $f_t, f_t' \sim p_t$ independently from the posterior. Let $\pi_t = \pi_{f_t}$ and $\pi'_t = \pi_{f'_t}$. \label{line:fdraw} \COMMENT{Two independent samples $f_t, f'_t$ from posterior}
			    \STATE Play iteration $t$ using $\pi_{t}$ for $h=1,\ldots,h_t-1$ and $\pi'_{t}$ for $h_t$ onwards.\label{line:one-step}
			    \STATE Update $S_t = S_{t-1} \cup \{x_t^h, a_t^h, r_t^h, x_t^{h+1}\}$ for $h = h_t$.
			 \ENDFOR
			 \RETURN $(\pi_1,\ldots,\pi_T)$.
		\end{algorithmic}
\end{algorithm}

At a high-level, the algorithm performs standard Thompson Sampling updates to the posterior given the observations, with three modifications. First is that we incorporate an optimistic bias in the distribution $p_t$ over $f$ (Line~\ref{line:outer-update}), similar to the FGTS approach. The second difference arises from the challenges in estimating the Bellman error, while the third is in how we sample from the posterior to obtain the agent's policy at each time. We now explain the latter two issues in detail.

\mypar{Inner loop to estimate Bellman error.} Note that we would ideally define the \emph{likelihood function} as $\BE(f,f,x,\pi_f(x))^2$ for any $x$, but this requires a conditional expectation over the sampling of the next state $x'\sim P(\cdot\mid x,\pi_f(x))$. Replacing the expectation with a single random sample suffers from bias due to the double sampling issue~\citep{antos2008learning}, which arises from the non-linearity of squared loss inside the expectation. An ideal solution to this, following~\citet{antos2008learning} is to define the cumulative likelihood as:
\begin{equation}
    \tilde{L}_t(f) = \sum_{s=1}^t \Deltah_s^{h_s}(f,f)^2 - \min_{g\in\cF} \sum_{s=1}^t\Deltah_s^{h_s}(g,f)^2.
    \label{eq:ideal-likelihood}
\end{equation}
However, doing exact minimization over $g$ creates instability in the online learning analysis of the distributions $p_t$. To ameliorate this, we instead replace the optimization over $g$ with sampling from an appropriate distribution conditioned on $f$ (Line~\ref{line:inner-update}). This distribution favors functions $g$ which approximately minimize $\sum_{s=1}^t\Deltah_s^{h_s}(g,f)^2$. We then form a surrogate for the ideal likelihood of~\eqref{eq:ideal-likelihood} for each round $t$ in Line~\ref{line:likelihood}, where the second term acts as a soft minimization over $g$, given $f$. 

\mypar{Two samples to decouple roll-in and roll-out} Using the likelihood function, it is straightforward to define a Thompson sampling distribution $p_t$ over $f\in\cF$. Our definition in Line~\ref{line:outer-update} incorporates the optimistic term as well, giving higher weight to functions $f$ that predict large values in the initial step. This resembles both the design of FGTS~\citep{Zhang2021FeelGoodTS} as well as \olive~\citep{jiang2017contextual}. At this point, a typical TS approach would draw $f_t\sim p_t$ and act according to the resulting greedy policy. Doing so, however, creates a mismatch between the likelihood function we use to evaluate $f_t$ and the guarantees we need on it for learning. This issue is most easily seen if we imagine the distribution $p_t$ to be fixed across rounds (which is a reasonable intuition for a stable online learning algorithm). Then the likelihood at round $t$ contains samples collected at prior rounds, when we drew $f_s\sim p$ independently of $f_t$ for $s < t$. Thus we expect the likelihood of $f$ to approach $\rE_{f_s \sim p}\rE_{(x^h,a^h)\sim \pi_{f_s}}[\BE(f,f,x^h, a^h)^2]$. However, when we roll-out according to $f_t$, we require $\rE_{f_s\sim p}\rE_{x^h\sim \pi_{f_s}}[\BE(f_t, f_t,x^h,\pi_{f_t}(x^h))^2]$ to be small. 

A similar discrepancy between the roll-in policy $a^h=\pi_{f_s}$ and the desired roll-out policy $a^h=\pi_{f_t}(x^h)$ is addressed in \olive by a one-step uniform exploration over a finite action space,
which is also adopted by subsequent works.
In this paper, we replace this one-step uniform exploration by a second, independent sample from the posterior, which allows us to perform one-step exploration in infinite action spaces without the need for our algorithm to know the underlying linear embedding. One main insight of this work is to demonstrate the effectiveness of such a \emph{two sample} strategy (Line~\ref{line:one-step}).
Specifically, we execute $\pi_t$ for the first $h_t$ time steps, and then complete the roll-out using $\pi'_t$ with $h_t\in[H]$ chosen uniformly, and we use only the samples from step $h_t$ in our likelihood at time $t$. The decoupling of the two samples is crucial to our analysis, although it will be interesting to investigate whether a single sample strategy can be analyzed using a different approach. 

\section{Main Results}
\label{sec:theory-main}

In this section, we present our main sample complexity guarantee for \alg. To do so, we need to introduce some measures of the complexity of our value function class $\cF$, which we do next.

\begin{definition}
For any $f \in \cF$, we define the set
$\cF(\epsilon, f)
=\left\{ g \in \cF:  \sup_{x,a,h} |\BE^h(g, f; x, a)| \leq \epsilon \right\}$ of functions that have a small Bellman error with $f$ for all $x,a$. Further assume that $\cF$ has an
$L_\infty$ cover $f_1,\ldots,f_N \in \cF$ for $N = N(\epsilon)$, so that $\forall f\in\cF$, $\min_i \sup_{x,a} |f(x,a) - f_i(x,a)|\leq \epsilon$. Then we define
$\kappa(\epsilon)= \sup_{f \in \cF}-\ln p_0(\cF(\epsilon, f))~\mbox{and}~ 
  \kappa'(\epsilon)= \ln N(\epsilon) $.
\label{def:kappa}
\end{definition}

If $\cF$ is finite with $|\cF| = N$, then we can choose $p_0$ to be uniform over $\cF$ to get $\kappa(\epsilon) = \kappa'(\epsilon) = \ln N$. We state our guarantee in terms of the regret of the learned policies with respect to the optimal policy $\pi_\star$~\eqref{eq:qstar}, that we define for a greedy policy $\pi_f$ with some $f\in\cF$ as:
\begin{equation*}
    \regret(f,x^1) = R(\pi_\star,x^1)-R(\pi_f,x^1), ~~\mbox{where}~~R(\pi,x^1)= \rE_{(x,a,r)\sim \pi|x^1} \sum_{h=1}^H r^h
\end{equation*}

The following result gives our main sample complexity bound for \alg. 

\begin{theorem}
    Under Assumptions~\ref{ass:realizable}-\ref{assumption:embedding}, suppose we run \alg (Alg.~\ref{alg:online_TS}) with parameters $\gamma =0.1$ and $\lambda \leq \eta \leq c/(\kappa(\epsilon) + \kappa'(\epsilon) + \ln T)$ for a universal constant $c$. Then choosing any $\epsilon = \order(1/T)$, we have 
    \begin{align*}
    \rE\; \sum_{t=1}^T \regret(f_t,x_t^1) &= \order\left(\frac{\kappa(\epsilon) + \kappa'(\epsilon)}{\lambda} + \lambda T + \tilde{\epsilon}(\lambda/\eta) T\right),
    \end{align*}
    where $\tilde{\epsilon}(\lambda/\eta) = \inf_{\mu>0} \left[ 8(\lambda/\eta) \dc(\epsilon_2)H\mu^{2} + 2\mu H \epsilon_2 B_2 + \mu^{-1} \br(\epsilon_1) +\epsilon_1 H B_1    \right]$.
\label{thm:main}
\end{theorem}

Note that the bound above is not a bound on the online regret of \alg, since the policy we execute at round $t$ is not $\pi_t$, but a mixture of $\pi_t$ and $\pi'_t$ (line~\ref{line:one-step} of Algorithm~\ref{alg:online_TS}). However, the bound implies a PAC guarantee as we can choose a policy $\pi_1,\ldots,\pi_T$ uniformly, and use a standard online-to-batch conversion~\citep{cesa2004generalization} to obtain a regret bound for the returned policy. 

To interpret the general guarantee of Theorem~\ref{thm:main}, we now consider some special cases where the parameters can be set optimally to simplify our result. We start with the well-studied setting of a finite function class, with Bellman rank and linear embedding dimensions being finite as well.

\begin{corollary}[Finite dimensional embeddings and finite $\cF$] Under conditions of Theorem~\ref{thm:main}, assume further that $|\cF| = N < \infty$, $\br(0) \leq d_1$ and $\dc(0) \leq d_2$. Then \alg with $p_0$ as uniform on $\cF$, $\gamma = 0.1$, $\eta = \order(1/\ln (NT))$ and $\lambda = T^{-3/4}H\,(d_1^2d_2)^{-1/4} (\ln (NT))^{1/2}$ returns a policy with an average regret at most $\order\left(H\,d_2^{1/4}\,\sqrt{\,d_1\ln (NT)}\,T^{-1/4}\right)$.
\label{cor:finite}
\end{corollary}

The bound has a favorable dependence on all the complexity parameters, with mild scaling in the horizon, dimensions and function class complexity. However, the average regret decays at a rate of $T^{-1/4}$. Consequently, the sample complexity to discover a policy with a regret at most $\epsilon$ is $\order\left(H^4(d_1 \ln (NT))^{2}\,|\cA|\,\epsilon^{-4}\right)$, which is suboptimal and slower than that of \olive in $\epsilon$ dependence. In the setting of a linear MDP, where $d_1 = d_2 = d$, the scaling with dimension is $d^3$. 

The main source of the suboptimality in $\epsilon$ is that our analysis relies on a decoupling argument that leverages Assumption~\ref{assumption:embedding} with an extra Cauchy-Schwarz inequality than in the typical analysis. The importance weighting over a uniform distribution in \olive does not lose an exponent in this step, and improving the analysis for our general scenario is an open direction for future work. Note however that all prior works that studied both finite action non-linear and infinite action (generalized) linear settings employ different algorithms for the two cases~\citep{jin2021bellman,du2021bilinear,foster2021statistical}, unlike our unified approach. When the linear embedding features are known, this suboptimality can be removed through an explicit experimental design in the feature space, as we show in the next section.

As a first example of our general result, we can instantiate the setting of Example~\ref{ex:comb}, where our sample complexity only scales with $\order(KL)$ instead of $\order(K^L)$. In the setting of Example~\ref{ex:basis}, we only depend on the size of the (potentially unknown) basis. 

Under the infinite-dimensional examples of Proposition~\ref{prop:spectral}, we note that the geometric case is almost identical to the finite dimensional setting up to log factors, since the \emph{effective rank} only scales as $\log(1/\epsilon)$, so that we can always make the additional terms coming from $\epsilon$ to be lower order. The polynomial case does yield qualitatively different results, which we show next.

\begin{corollary}[Polynomial spectral decay and finite $\cF$] Under conditions of Theorem~\ref{thm:main}, assume further that $|\cF| = N < \infty$, $B_1 = B_2 = B$ and $\br(\epsilon),\dc(\epsilon) \leq H\,d_q\epsilon^{-2q/(1-q)}$ for some $q \in (0,1)$. Then \alg with $p_0$ as uniform on $\cF$, $\gamma = 1/36$, along with appropriate settings of $\eta$ and $\lambda$ returns a policy with an average regret at most $\order\left(H(\ln (NT))^{\nicefrac{(1-q)^2}{2}}d^{\nicefrac{(3-q)(1-q)}{4}}\,T^{\nicefrac{-(1-q)^2}{4}} B^{\nicefrac{q(3-q)}{2}}\right)$.
\label{cor:poly}
\end{corollary}
The result matches that of Corollary~\ref{cor:finite} for $q = 0$ corresponding to the case of finite dimensions. More generally, we allow scaling to infinite dimensional states and actions, as long as the Bellman rank and linear embedding assumptions have a low intrinsic dimension. 

Our results have a qualitatively similar flavor to those of \cite{Zhang2021FeelGoodTS}, but for two important differences. \cite{Zhang2021FeelGoodTS} do not need to account for the residual variance term in the likelihood due to the deterministic dynamics assumption, and hence they do not incur the loss in rates that we suffer. They also do not work under the low Bellman rank assumption on the problem, which significantly limits the class of problems where their guarantees hold. 

We finally illustrate the benefits of our contextual formulation by studying mixtures of MDPs with a small Bellman rank.

\begin{corollary}[Mixture of low-rank MDPs]
    Consider a collection $M$ different MDPs $\{\cM_i\}_{i=1}^M$ over the same state and action space, each with a Bellman rank at most $d_1$ and linear embedding dimension $d_2$. Let $x^1\sim \cD$ where $\cD$ is uniform over $[M]$, with the subsequent transitions happening according $\cM_i$. Suppose $|\cF| = N$. Then under the parameter settings of Corollary~\ref{cor:finite}, \alg returns a policy with a regret at most $\order\left(H\,d_2^{1/4}\,\sqrt{\,d_1\ln (NT)}\,T^{-1/4}\right)$.
    \label{cor:contextual}
\end{corollary}

Comparing Corollaries~\ref{cor:contextual} and~\ref{cor:finite}, we observe that the mixture setting poses no extra challenge. In contrast, the mixture problem has a Bellman rank scaling with $M$, as we need to concatenate the respective embeddings for each $\cM_i$, causing a na\"ive application of \olive to incur an extra $M^2$ term in the sample complexity. We believe the difference arises since the average Bellman error of \olive mixes samples across different contexts, while we use the squared Bellman error which can leverage the structure for individual transitions more effectively. Of course, this comes at the price of an additional completeness assumption, which would be interesting to eliminate in future work. 


We now describe an improvement to our results for the setting of known features in Assumption~\ref{assumption:embedding}.

\section{Experimental Design for Known Linear Embedding Features}
\label{sec:design}

In this section, we study a special case of our setting where the features $\phi(x,a)$ in Assumption~\ref{assumption:embedding} are known, and finite dimensional, with $\phi(x,a) \in \R^{d_2}$. For this setting, we consider an adaptation of \alg which replaces the two sample strategy with experimental design in the feature space. The algorithm is presented in Algorithm~\ref{alg:online_TS_design}.

\begin{algorithm}[tb]
	\caption{\algdesignlong (\algdesign)}
	\label{alg:online_TS_design}
		\begin{algorithmic}[1]
		    \REQUIRE Function class $\cF$, prior $p_0\in\Delta(\cF)$, learning rates $\eta, \gamma$ and optimism coefficient $\lambda$.
		    \STATE Set $S_0 = \emptyset$.
			\FOR{$t=1,\ldots,T$}
			    \STATE Observe $x_t^1 \sim \cD$ and draw $h_t \sim \{1,\ldots,H\}$ uniformly at random.
			    \STATE Define $q_t(g) = p(g | f,S_{t-1}) \propto p_0(g)\exp(-\gamma\sum_{s=1}^{t-1} \Deltah_s^{h_s}(g,f)^2)$.\COMMENT{Inner loop TS update}
			    \STATE Define $L_t^h(f) = \eta\Deltah_t^h(f,f)^2 + \frac{\eta}{\gamma}\ln\rE_{g\sim q_t}\left[\exp(-\gamma\Deltah_t^h(g,f)^2)\right]$.\COMMENT{Likelihood function}
			    \STATE Define $p_t(f) = p(f|S_{t-1}) \propto p_0(f)\exp(\sum_{s=1}^{t-1}(\lambda f(x_s^1) - L_s^{h_s}(f))$ as the posterior. \COMMENT{Outer loop Optimistic TS update}
			    \STATE Draw $f_t \sim p_t$ from the posterior. Let $\pi_t = \pi_{f_t}$ and execute $a_t^h=\pi_t(x_t^h)$ for $h=1,\ldots,h_t-1$ to observe $x_t^{h_t}$. \label{line:fdraw-des}
			    \STATE Let $\rho_t\in\Delta(\cA)$ be a $G$-optimal design for $\phi(x_t^{h_t},a)_{a\in\cA}$ (Equation~\ref{eq:g-opt}). Draw \mbox{$a^t_{h_t} \sim \rho_t$} and observe $r_t^{h_t}$ and $x_t^{h_t+1}$.\COMMENT{$G$-optimal design}\label{line:adraw}
			    \STATE Update $S_t = S_{t-1} \cup \{x_t^h, a_t^h, r_t^h, x_t^{h+1}\}$ for $h = h_t$.
			 \ENDFOR
			 \RETURN $(\pi_1,\ldots,\pi_T)$.
		\end{algorithmic}
\end{algorithm}

Before we delve into the pseudocode, recall that the $G$-optimal design, given a set of vectors $\{\phi(x,a)\}_{a\in\cA}$ is a distribution $\rho(x)\in\Delta(\cA)$ over the action space, given by~\citep[see e.g.][]{fedorov2013theory,kiefer1960equivalence}.

\begin{equation}
    \rho(x) = \argmin_{\rho\in\Delta(\cA)}\underbrace{\max_{a\in\cA} \|\phi(x,a)\|_{\Sigma_x(\rho)^{-1}}^2}_{g(x,\rho)} \quad\mbox{where}\quad \Sigma_x(\rho) = \rE_{a\sim\rho} \phi(x,a)\phi(x,a)^\top.
    \label{eq:g-opt}
\end{equation}
If $\Sigma_x(\rho)$ is not full rank, then we can replace $\Sigma_x(\rho)^{-1}$ by the corresponding pseudo-inverse. 
Furthermore, by the Kiefer-Wolfowitz theorem~\citep{kiefer1960equivalence}, it is known that the design $\rho(x)$ satisfies $g(x,\rho(x)) =\text{rank}(\{\phi(x,a):a \in \cA\}) \leq d_2$. The criterion $g(\rho)$ corresponds to worst prediction variance at some action $a$, of an ordinary least squares estimator given samples drawn from $\rho(x)$. For intuition in the finite action setting, where $\phi(x,a) = e_a$, $\rho(x)$ corresponds to a uniform distribution over the action set and $d_2 = |\cA|$ is the variance of importance sampling. 

With this context, \algdesign is a relatively natural adaptation of \alg, when $\phi(x,a)$ is a known feature map. Concretely, we no longer use the two sample scheme. Instead, we only draw one function $f_t\sim p_t$ in Line~\ref{line:fdraw-des} and execute the first $h_t-1$ actions using the corresponding greedy policy. Having observed $x_t^{h_t}$, we now choose the action at time $t$ using a $G$-optimal design in the feature space for \emph{this particular state} in Line~\ref{line:adraw}. That is, the action sampling distribution $\rho_t$ corresponds to $\rho(x^{h_t}_t)$ from Equation~\ref{eq:g-opt}, and the design is done individually at each state. We then observe the reward and next transition for this action, which gets recorded into our dataset as before. 

We have the following general guarantee for \algdesign.

\begin{theorem}
    Under Assumptions~\ref{ass:realizable}-\ref{assumption:embedding}, suppose further that $\dc(0) = d_2$ and the features $\phi^h(x,a)$ are known. Suppose we run \algdesign (Alg.~\ref{alg:online_TS_design}) with parameters $\gamma =0.1$ and $\lambda \leq \eta \leq c/(\kappa(\epsilon) + \kappa'(\epsilon) + \ln T)$ for a universal constant $c$. Then choosing any $\epsilon = \order(1/T)$, we have 
    \begin{align*}
    \rE\; \sum_{t=1}^T \regret(f_t,x_t^1) &= \order\left(\frac{\kappa(\epsilon) + \kappa'(\epsilon)}{\lambda} + \lambda T + T\left(\epsilon_1 HB_1 + \frac{2\lambda\br(\epsilon_1)d_2 H}{\eta}\right)\right).
    \end{align*}
\label{thm:main-des}
\end{theorem}

Further optimizing over the choice of $\lambda$, we see that:
\begin{align*}
    \frac{1}{T}\rE\; \sum_{t=1}^T \regret(f_t,x_t^1) &= \order\left(\epsilon_1 HB_1 + (\kappa(\epsilon) + \kappa'(\epsilon) + \ln T) \sqrt{\frac{\br(\epsilon_1) d_2H}{T}}\right).
\end{align*}

Thus, in the setting of Corollary~\ref{cor:finite}, we immediately get a $1/\sqrt{T}$ rate in the above bound, leading to a $\order(d_1d_2H\ln^2 (NT)/\epsilon^2)$ sample complexity. This bound is superior or comparable to that for \olive in all the problem parameters, except for an extra dependency on $\ln N$. This extra dependency is caused by the two timescale learning rates (discrepancy between $\eta$ and $\gamma$) in the online minimax game analysis, which might be possible to improve in future work. For tabular problems with $S$ states and $A$ actions, where $d_1 = S$, $d_2 = A$ and $\ln N = \tilde{\order}(SAH)$, the bounds for \algdesign as well as \olive scale as $S^3A^2$, with an additional $H^2$ term in $\olive$ compared to the bound of \algdesign. 

For linear MDPs with finite dimensional features, $d_1 = d_2 = d$ and $\ln N = \order(d)$, so that our bound scales as $\order(d^4)$, which is a factor of $d$ worse relative to  LSVI-UCB~\citep{jin2020provably}, and a factor of $d^2$ worse relatively to \citep{zanette2020learning}. 
One reason for the suboptimality in $d$ is due to our choice of $V$-type Bellman rank instead of $Q$-type Bellman rank to allow non-linear scenarios. For the nicer setting of $Q$-type Bellman rank, an analysis of a single sample based Thompson sampling is recently carried out in \cite{DMZZ2021-neurips}, leading to a result that matches that of \citep{zanette2020learning} for linear MDPs. 
Another reason for the suboptimality is because our MDP model allows long range contextual dependency on  $x_t^1$, which is not allowed in most prior works. If we remove this dependency by considering only non-contextual MDP models, then we can replace the online regreet analysis of this paper by the uniform convergence analysis of \citep{DMZZ2021-neurips}. Doing so avoids the slow-fast learning rate issue in our online minimax analysis and implies a sample complexity of $\order(d_1d_2H\ln (NT)/\epsilon^2)$. However, the technique cannot be used to analyze double-posterior sampling, and thus we do not consider it in this paper.

Note that the recent work of~\citet{foster2021statistical} also analyzes design-based approaches for model based RL, but both their problem setting and algorithmic details differ significantly from ours. The result of Theorem~\ref{thm:main-des} demonstrates that the suboptimality in $\epsilon$ in our more general result of Theorem~\ref{thm:main} stems purely from the harder setting of unknown linear embedding features. We leave the development of corresponding results for spectral decay and combinatorial action settings to the reader. 

In terms of analysis, the only change is that we are able to do different handling of a decoupling step using the property of optimal design in the analysis of Theorem~\ref{thm:main-des}, and the error terms in the linear embedding can be ignored due to the finite dimensional assumption. The rest of our arguments stay the same, and we provide a proof in Appendix~\ref{sec:proof-des}.

\section{Proof Sketch of Theorem~\ref{thm:main}}
\label{sec:proof-main}

In this section, we provide the analysis for our main result on the sample complexity of \alg. To begin, we recall a standard result on the online regret of any value-based RL algorithm.

\begin{proposition}[\citet{jiang2017contextual}]
 Given any $f\in\cF$ and $x^1\in\cX^1$, we have 
 \begin{align*}
  \regret(f,x^1) =&   \sum_{h=1}^H  \; \rE_{(x^h,a^h) \sim \pi_{f}|x^1}  \BE^{h}(f,f,x^h,a^h)   -  \Delta f(x^1), ~~ \mbox{where}~~\Delta f(x^1)= f(x^1) - Q^1_\star(x^1).
  \end{align*}
  \label{prop:value-decomposition}
\end{proposition}

The proposition is a consequence of a simple telescoping argument. The RHS looks related to the likelihood function in our update (line~\ref{line:outer-update} in Algorithm~\ref{alg:online_TS}), but there are a few important differences. First, we do not have access to the Bellman error, which is instead approximated through difference of the TD term $\Deltah_t(f,f)^2$ and the residual variance measured using $\Deltah_t(g,f)^2$. If the posterior of $g$ concentrates around the Bellman projection of $f$ onto $\cF$, then we can expect our likelihood to resemble $\BE^h(f,f,x^h,a^h)^2$. The presence of a squared term instead of the linear dependence in Proposition~\ref{prop:value-decomposition} is typical to algorithms which use completeness~\citep{jin2021bellman}. However, there are two more significant issues. On the RHS of Proposition~\ref{prop:value-decomposition}, the function $f$ whose Bellman error is measured is identical to the one whose greedy policy picks all the actions. On the other hand, in our algorithm, we control the regression loss of a function $f$ that is different from the roll-in policy. For non-linear functions $f$, prior works~\citep{jiang2017contextual,agarwal2020flambe} control this change in measure over the states using the low-rank property, while the distribution of the final action is corrected by importance weighting over the finite action set. For linear functions, change of measure over both $x^h$ and $a^h$ can be done using a similar use of the low-rank property, without explicit weighting over the actions~\citep{jin2020provably,du2021bilinear,jin2021bellman}. 

In our case, we adopt a slightly different approach to analyze the Bellman error term in Proposition~\ref{prop:value-decomposition}. We first apply the low-rank property to \emph{decouple} the roll-in policy from the $f$ being evaluated. We subsequently use the linear embeddability assumption to \emph{decouple} the action selection at step $h$. This part of our analysis resembles that of~\citep{Zhang2021FeelGoodTS} for the contextual bandit setting, which is the reason we adopt the name \emph{decoupling coefficient} as their work, for the measure of the linear embeddability dimension. We call these two results decoupling lemmas, which can be found in Appendix~\ref{sec:decoupling}.  Using the two decoupling coefficients, we can obtain the following result.

\begin{proposition}[Decoupling]
  We have
  \begin{align*}
    \lambda \rE \; \regret(f_t,x_t^1)\leq \rE \; \rE_{f|S_{t-1}} \left[-\lambda \Delta f(x_t^1)
    + 0.5 \eta \BE^{h_t}(f,f,x_t^{h_t},a_t^{h_t})^2\right]
    +\lambda \tilde{\epsilon}(\lambda/\eta),
  \end{align*}
  where $\tilde{\epsilon}(\lambda/\eta) = \inf_{\mu>0} \left[ 8(\lambda/\eta) \dc(\epsilon_2)H\mu^{2} + 2\mu H \epsilon_2 B_2 + \mu^{-1} \br(\epsilon_1) +\epsilon_1 H B_1    \right]$.
  \label{lem:decouple}
\end{proposition}

For optimal parameter settings, the bound scales with $\sqrt{\dc(\epsilon_2)\rE\rE_{f|S_{t-1}}\BE(f,f;x_t^{h_t},a_t^{h_t})^2}$, and this square root is responsible for our $O(\epsilon^{-4})$ rate.
In contrast with Proposition~\ref{prop:value-decomposition}, Proposition~\ref{lem:decouple} measures the squared Bellman error of functions $f$ according to the states $x_t^h$ and actions $a_t^h$ observed during the algorithm's execution, and which we can hope to control if the \alg updates converge to their respective optima for both the time scales. The remainder of our analysis focuses on this online learning component, details of which are presented in Appendix~\ref{sec:online}. We now give our main result to control the regret of the online learning process.

\begin{proposition}[Outer loop convergence]
Assume that  $\gamma=0.1$ and $\eta \leq 0.01$. Then we have
\begin{align*}
&-\sum_{t=1}^T \rE  \rE_{f|S_{t-1}}
\lambda \Delta f(x_t^1) + \eta
(1-6\eta) e^{-3\eta(1-6\eta)} \sum_{t=1}^T \rE \rE_{f|S_{t-1}} \BE^{h_t}(f,f;x_t^{h_t},a_t^{h_t})^2\\
\leq &
\eta (1+6\eta) e^{3\eta(1+6\eta)} \sum_{t=1}^T \rE \rE_{f,g|S_{t-1}} \BE^{h_t}(g,f;x_t^{h_t},a_t^{h_t})^2
+ \lambda \epsilon T +4
 \eta T \epsilon^2 +  \kappa(\epsilon) + 1.5 \lambda^2 T  .
\end{align*}
\label{lem:regret-hatc-Z}
\end{proposition}

We prove Proposition~\ref{lem:regret-hatc-Z} in Appendix~\ref{sec:online-outer}.
The LHS of the proposition qualitatively resembles the RHS of Propositoin~\ref{lem:decouple}. Indeed, we subsequently set the constants so that they match. This means that the regret in Proposition~\ref{prop:value-decomposition} can be further upper bounded by the RHS of Proposition~\ref{lem:regret-hatc-Z}. The first term in the bound is intuitive. It bounds the Bellman residual of $f$ in terms of quality of the functions $g\sim p(\cdot \mid f,S_{t-1})$ in capturing the Bellman operator applied to $f$. If the inner loop of \alg converges at a reasonable rate, we expect this error to be small by the completeness assumption. The next three terms on the RHS are a bound on the log-partition function in our outer loop updates, which are controlled by using typical potential function arguments in the analysis of multiplicative updates. The final term can be controlled by appropriate setting of the optimism parameter $\lambda$.

We now give the main bound on the convergence of our inner updates to control the first term on the RHS of Proposition~\ref{lem:regret-hatc-Z}. 

\begin{proposition}
Assume that $\gamma=0.1$. 
Given any absolute constant $c_1>0$,
there exists an absolute constant $c_0$ such that when
$c_0 \eta  (\kappa(\epsilon)+\kappa'(\epsilon) +\ln T) \leq 0.5 c_1 \gamma < 0.2\gamma$ and $\lambda\leq\eta$, then
\begin{align*}
\rE\; & \sum_{t=1}^T \rE_{f,g|S_{t-1}} \BE^{h_t}(g,f;x_t^{h_t},a_t^{h_t})^2 \leq \order\Bigg(c_0   (\epsilon T + \kappa(\epsilon)+\kappa'(\epsilon)+1) 
+  c_1 (\lambda/\eta) \tilde{\epsilon}(\lambda/\eta) T
\\
& \left.\qquad \qquad + c_1 \sum_{t=1}^T \rE \; \rE_{f|S_{t-1}} |\BE^{h_t}(f,f;x_t^{h_t},a_t^{h_t})|^2
 + c_1 (\lambda/\eta) \sum_{t=1}^T \rE \;\regret(f_t,x_t^1)\right) ,
\end{align*}
where $\tilde{\epsilon}(\lambda/\eta)$ is defined in Proposition~\ref{lem:decouple}.
\label{lem:inner}
\end{proposition}

The proof of Proposition~\ref{lem:inner} is rather long and technical. For the reader's convenience, we first prove a simpler result which yields a worse $\order(\epsilon^{-8})$ sample complexity in Appendix~\ref{sec:online-inner-slow}. We then show how to improve the bound to obtain Proposition~\ref{lem:inner} in Appendix~\ref{sec:online-inner}.

With these results, we set both $c_1$ and $\eta$ sufficiently small so that $(1-6\eta)\exp(-3(1-6\eta)) - O(c_1 (1+6\eta) e^{3\eta(1+6\eta)}) \geq 0.5$,
$O(c_1 (1+6\eta) e^{3\eta(1+6\eta)}) \leq 0.5$,
along with the stated values of $\epsilon$ and $\gamma$. Plugging these into the bounds of our intermediate results and simplifying gives the result of Theorem~\ref{thm:main}. Finally, we summarize the proof of Corollary~\ref{cor:finite} and defer that of Corollary~\ref{cor:poly} to Appendix~\ref{sec:cor}.

\begin{proof}[Proof of Corollary~\ref{cor:finite}]
    We set the parameters as follows. Since $\epsilon_1 = \epsilon_2 = 0$, we minimize over $\mu$ to get $\mu = \left(d_1\eta/(d_2\lambda H)\right)^{1/3}$. We now optimize over the choice of $\eta$, for which the leading order terms are $\eta^2 T\ln N/\lambda + T\tilde{\epsilon}(\lambda/\eta)$. Then optimizing for $\lambda$ by including the $\ln N/\lambda$ term yields the stated guarantee.
\end{proof}

\section{Conclusion and Discussion}

In this paper, we combine the framework of low Bellman rank with a linear embedding assumption over the action space to introduce a new class of rich problems which encompasses all settings with finite actions and linear function approximation, and enables new ones such as combinatorial action spaces. We show that \alg solves this class of problems under the usual completeness and realizability assumptions on the value function class. We believe that the identification of this linear embedding structure over actions as the key enabler of \emph{one-step exploration} in the action space has the potential to apply to broader algorithmic approaches beyond \alg.

The most immediate direction for future work is to improve our sample complexity results by sharpening our decoupling results. Understanding if similar results are attainable without completeness is another challenge. Finally, it would be interesting to understand what structures beyond linear embeddability afford sample-efficiency, and study applications of these ideas to continuous control problems.

\section*{Acknowledgements}

We would like to thank Dylan Foster, Jian Qian and Sasha Rakhlin for identifying an error in the original proof of Proposition~\ref{lem:inner}, which we have subsequently addressed.

\bibliographystyle{plainnat}
\bibliography{myrefs}

\begin{thebibliography}{57}
\providecommand{\natexlab}[1]{#1}
\providecommand{\url}[1]{\texttt{#1}}
\expandafter\ifx\csname urlstyle\endcsname\relax
  \providecommand{\doi}[1]{doi: #1}\else
  \providecommand{\doi}{doi: \begingroup \urlstyle{rm}\Url}\fi

\bibitem[Abbasi-Yadkori and Neu(2014)]{abbasi2014online}
Yasin Abbasi-Yadkori and Gergely Neu.
\newblock Online learning in mdps with side information.
\newblock \emph{arXiv preprint arXiv:1406.6812}, 2014.

\bibitem[Agarwal et~al.(2020{\natexlab{a}})Agarwal, Henaff, Kakade, and
  Sun]{agarwal2020pc}
Alekh Agarwal, Mikael Henaff, Sham Kakade, and Wen Sun.
\newblock {PC-PG}: Policy cover directed exploration for provable policy
  gradient learning.
\newblock \emph{Advances in neural information processing systems},
  2020{\natexlab{a}}.

\bibitem[Agarwal et~al.(2020{\natexlab{b}})Agarwal, Kakade, Krishnamurthy, and
  Sun]{agarwal2020flambe}
Alekh Agarwal, Sham Kakade, Akshay Krishnamurthy, and Wen Sun.
\newblock Flambe: Structural complexity and representation learning of low rank
  mdps.
\newblock \emph{Advances in neural information processing systems},
  2020{\natexlab{b}}.

\bibitem[Agarwal et~al.(2019)Agarwal, Bullins, Hazan, Kakade, and
  Singh]{agarwal2019online}
Naman Agarwal, Brian Bullins, Elad Hazan, Sham Kakade, and Karan Singh.
\newblock Online control with adversarial disturbances.
\newblock In \emph{International Conference on Machine Learning}, pages
  111--119. PMLR, 2019.

\bibitem[Agrawal and Jia(2017)]{agrawal2017posterior}
Shipra Agrawal and Randy Jia.
\newblock Posterior sampling for reinforcement learning: worst-case regret
  bounds.
\newblock In \emph{Advances in Neural Information Processing Systems}, pages
  1184--1194, 2017.

\bibitem[Antos et~al.(2008)Antos, Szepesv{\'a}ri, and Munos]{antos2008learning}
Andr{\'a}s Antos, Csaba Szepesv{\'a}ri, and R{\'e}mi Munos.
\newblock Learning near-optimal policies with bellman-residual minimization
  based fitted policy iteration and a single sample path.
\newblock \emph{Machine Learning}, 71\penalty0 (1):\penalty0 89--129, 2008.

\bibitem[Asadi et~al.(2021)Asadi, Parikh, Parr, Konidaris, and
  Littman]{asadi2021deep}
Kavosh Asadi, Neev Parikh, Ronald~E Parr, George~D Konidaris, and Michael~L
  Littman.
\newblock Deep radial-basis value functions for continuous control.
\newblock In \emph{Proceedings of the Thirty-Fifth AAAI Conference on
  Artificial Intelligence}, pages 6696--6704, 2021.

\bibitem[Bertsekas and Tsitsiklis(1996)]{bertsekas1996neuro}
Dimitri~P. Bertsekas and John~N. Tsitsiklis.
\newblock \emph{Neuro-Dynamic Programming}.
\newblock Athena Scientific, 1st edition, 1996.
\newblock ISBN 1886529108.

\bibitem[Borkar(2009)]{borkar2009stochastic}
Vivek~S Borkar.
\newblock \emph{Stochastic approximation: a dynamical systems viewpoint},
  volume~48.
\newblock Springer, 2009.

\bibitem[Cesa-Bianchi and Lugosi(2012)]{cesa2012combinatorial}
Nicolo Cesa-Bianchi and G{\'a}bor Lugosi.
\newblock Combinatorial bandits.
\newblock \emph{Journal of Computer and System Sciences}, 78\penalty0
  (5):\penalty0 1404--1422, 2012.

\bibitem[Cesa-Bianchi et~al.(2004)Cesa-Bianchi, Conconi, and
  Gentile]{cesa2004generalization}
Nicolo Cesa-Bianchi, Alex Conconi, and Claudio Gentile.
\newblock On the generalization ability of on-line learning algorithms.
\newblock \emph{IEEE Transactions on Information Theory}, 50\penalty0
  (9):\penalty0 2050--2057, 2004.

\bibitem[Dai et~al.(2018)Dai, Shaw, Li, Xiao, He, Liu, Chen, and
  Song]{dai2018sbeed}
Bo~Dai, Albert Shaw, Lihong Li, Lin Xiao, Niao He, Zhen Liu, Jianshu Chen, and
  Le~Song.
\newblock Sbeed: Convergent reinforcement learning with nonlinear function
  approximation.
\newblock In \emph{International Conference on Machine Learning}, pages
  1125--1134. PMLR, 2018.

\bibitem[Dann(2018)]{dann2018personal}
Christoph Dann.
\newblock personal communication, 2018.

\bibitem[Dann et~al.(2021)Dann, Mohri, Zhang, and Zimmert]{DMZZ2021-neurips}
Christoph Dann, Mehryar Mohri, Tong Zhang, and Julian Zimmert.
\newblock A provably efficient model-free posterior sampling method for
  episodic reinforcement learning.
\newblock In \emph{Neurips}, 2021.
\newblock URL \url{papers/neurips21-rl.pdf}.

\bibitem[Dean et~al.(2020)Dean, Mania, Matni, Recht, and Tu]{dean2020sample}
Sarah Dean, Horia Mania, Nikolai Matni, Benjamin Recht, and Stephen Tu.
\newblock On the sample complexity of the linear quadratic regulator.
\newblock \emph{Foundations of Computational Mathematics}, 20\penalty0
  (4):\penalty0 633--679, 2020.

\bibitem[Du et~al.(2019)Du, Krishnamurthy, Jiang, Agarwal, Dudik, and
  Langford]{du2019provably}
Simon Du, Akshay Krishnamurthy, Nan Jiang, Alekh Agarwal, Miroslav Dudik, and
  John Langford.
\newblock Provably efficient rl with rich observations via latent state
  decoding.
\newblock In \emph{International Conference on Machine Learning}, pages
  1665--1674. PMLR, 2019.

\bibitem[Du et~al.(2021)Du, Kakade, Lee, Lovett, Mahajan, Sun, and
  Wang]{du2021bilinear}
Simon~S Du, Sham~M Kakade, Jason~D Lee, Shachar Lovett, Gaurav Mahajan, Wen
  Sun, and Ruosong Wang.
\newblock Bilinear classes: A structural framework for provable generalization
  in rl.
\newblock \emph{International Conference on Machine Learning}, 2021.

\bibitem[Fedorov(2013)]{fedorov2013theory}
V.V. Fedorov.
\newblock \emph{Theory Of Optimal Experiments}.
\newblock Probability and Mathematical Statistics. Elsevier Science, 2013.
\newblock ISBN 9780323162463.
\newblock URL \url{https://books.google.com/books?id=PwUz-uXnImcC}.

\bibitem[Feng et~al.(2021)Feng, Yin, Agarwal, and Yang]{feng2021provably}
Fei Feng, Wotao Yin, Alekh Agarwal, and Lin Yang.
\newblock Provably correct optimization and exploration with non-linear
  policies.
\newblock In \emph{International Conference on Machine Learning}, pages
  3263--3273. PMLR, 2021.

\bibitem[Foster et~al.(2021)Foster, Kakade, Qian, and
  Rakhlin]{foster2021statistical}
Dylan~J Foster, Sham~M Kakade, Jian Qian, and Alexander Rakhlin.
\newblock The statistical complexity of interactive decision making.
\newblock \emph{arXiv preprint arXiv:2112.13487}, 2021.

\bibitem[Hallak et~al.(2015)Hallak, Di~Castro, and
  Mannor]{hallak2015contextual}
Assaf Hallak, Dotan Di~Castro, and Shie Mannor.
\newblock Contextual markov decision processes.
\newblock \emph{arXiv preprint arXiv:1502.02259}, 2015.

\bibitem[Hao et~al.(2021)Hao, Lattimore, Szepesv{\'a}ri, and
  Wang]{hao2021online}
Botao Hao, Tor Lattimore, Csaba Szepesv{\'a}ri, and Mengdi Wang.
\newblock Online sparse reinforcement learning.
\newblock In \emph{International Conference on Artificial Intelligence and
  Statistics}, pages 316--324. PMLR, 2021.

\bibitem[Ie et~al.(2019)Ie, Jain, Wang, Narvekar, Agarwal, Wu, Cheng, Lustman,
  Gatto, Covington, et~al.]{ie2019reinforcement}
Eugene Ie, Vihan Jain, Jing Wang, Sanmit Narvekar, Ritesh Agarwal, Rui Wu,
  Heng-Tze Cheng, Morgane Lustman, Vince Gatto, Paul Covington, et~al.
\newblock Reinforcement learning for slate-based recommender systems: A
  tractable decomposition and practical methodology.
\newblock \emph{arXiv preprint arXiv:1905.12767}, 2019.

\bibitem[Jiang and Agarwal(2018)]{jiang2018open}
Nan Jiang and Alekh Agarwal.
\newblock Open problem: The dependence of sample complexity lower bounds on
  planning horizon.
\newblock In \emph{Conference On Learning Theory}, pages 3395--3398. PMLR,
  2018.

\bibitem[Jiang et~al.(2017)Jiang, Krishnamurthy, Agarwal, Langford, and
  Schapire]{jiang2017contextual}
Nan Jiang, Akshay Krishnamurthy, Alekh Agarwal, John Langford, and Robert~E
  Schapire.
\newblock Contextual decision processes with low bellman rank are
  pac-learnable.
\newblock In \emph{International Conference on Machine Learning}, pages
  1704--1713. PMLR, 2017.

\bibitem[Jin et~al.(2020)Jin, Yang, Wang, and Jordan]{jin2020provably}
Chi Jin, Zhuoran Yang, Zhaoran Wang, and Michael~I Jordan.
\newblock Provably efficient reinforcement learning with linear function
  approximation.
\newblock In \emph{Conference on Learning Theory}, pages 2137--2143. PMLR,
  2020.

\bibitem[Jin et~al.(2021)Jin, Liu, and Miryoosefi]{jin2021bellman}
Chi Jin, Qinghua Liu, and Sobhan Miryoosefi.
\newblock Bellman eluder dimension: New rich classes of rl problems, and
  sample-efficient algorithms.
\newblock \emph{Advances in neural information processing systems}, 2021.

\bibitem[Kakade et~al.(2020)Kakade, Krishnamurthy, Lowrey, Ohnishi, and
  Sun]{kakade2020information}
Sham~M. Kakade, Akshay Krishnamurthy, Kendall Lowrey, Motoya Ohnishi, and Wen
  Sun.
\newblock Information theoretic regret bounds for online nonlinear control.
\newblock In \emph{NeurIPS}, 2020.

\bibitem[Kiefer and Wolfowitz(1960)]{kiefer1960equivalence}
Jack Kiefer and Jacob Wolfowitz.
\newblock The equivalence of two extremum problems.
\newblock \emph{Canadian Journal of Mathematics}, 12:\penalty0 363--366, 1960.

\bibitem[Mania et~al.(2019)Mania, Tu, and Recht]{mania2019certainty}
Horia Mania, Stephen Tu, and Benjamin Recht.
\newblock Certainty equivalence is efficient for linear quadratic control.
\newblock \emph{arXiv preprint arXiv:1902.07826}, 2019.

\bibitem[Mania et~al.(2020)Mania, Jordan, and Recht]{mania2020active}
Horia Mania, Michael~I Jordan, and Benjamin Recht.
\newblock Active learning for nonlinear system identification with guarantees.
\newblock \emph{arXiv preprint arXiv:2006.10277}, 2020.

\bibitem[Mhammedi et~al.(2020)Mhammedi, Foster, Simchowitz, Misra, Sun,
  Krishnamurthy, Rakhlin, and Langford]{mhammedi2020learning}
Zakaria Mhammedi, Dylan~J Foster, Max Simchowitz, Dipendra Misra, Wen Sun,
  Akshay Krishnamurthy, Alexander Rakhlin, and John Langford.
\newblock Learning the linear quadratic regulator from nonlinear observations.
\newblock \emph{arXiv preprint arXiv:2010.03799}, 2020.

\bibitem[Misra et~al.(2019)Misra, Henaff, Krishnamurthy, and
  Langford]{misra2019kinematic}
Dipendra Misra, Mikael Henaff, Akshay Krishnamurthy, and John Langford.
\newblock Kinematic state abstraction and provably efficient rich-observation
  reinforcement learning.
\newblock \emph{arXiv preprint arXiv:1911.05815}, 2019.

\bibitem[Modi et~al.(2018)Modi, Jiang, Singh, and Tewari]{modi2018markov}
Aditya Modi, Nan Jiang, Satinder Singh, and Ambuj Tewari.
\newblock Markov decision processes with continuous side information.
\newblock In \emph{Algorithmic Learning Theory}, pages 597--618. PMLR, 2018.

\bibitem[Modi et~al.(2021)Modi, Chen, Krishnamurthy, Jiang, and
  Agarwal]{modi2021model}
Aditya Modi, Jinglin Chen, Akshay Krishnamurthy, Nan Jiang, and Alekh Agarwal.
\newblock Model-free representation learning and exploration in low-rank mdps.
\newblock \emph{arXiv preprint arXiv:2102.07035}, 2021.

\bibitem[Osband et~al.(2013)Osband, Russo, and Van~Roy]{osband2013more}
Ian Osband, Daniel Russo, and Benjamin Van~Roy.
\newblock (more) efficient reinforcement learning via posterior sampling.
\newblock \emph{arXiv preprint arXiv:1306.0940}, 2013.

\bibitem[Osband et~al.(2016)Osband, Van~Roy, and Wen]{osband2016generalization}
Ian Osband, Benjamin Van~Roy, and Zheng Wen.
\newblock Generalization and exploration via randomized value functions.
\newblock In \emph{International Conference on Machine Learning}, pages
  2377--2386. PMLR, 2016.

\bibitem[Puterman(2014)]{puterman2014markov}
M.L. Puterman.
\newblock \emph{Markov Decision Processes: Discrete Stochastic Dynamic
  Programming}.
\newblock Wiley Series in Probability and Statistics. Wiley, 2014.
\newblock ISBN 9781118625873.
\newblock URL \url{https://books.google.com/books?id=VvBjBAAAQBAJ}.

\bibitem[Rajeswaran et~al.(2017)Rajeswaran, Lowrey, Todorov, and
  Kakade]{rajeswaran2017towards}
Aravind Rajeswaran, Kendall Lowrey, Emanuel~V Todorov, and Sham~M Kakade.
\newblock Towards generalization and simplicity in continuous control.
\newblock \emph{Advances in Neural Information Processing Systems}, 30, 2017.

\bibitem[Russo(2019)]{russo2019worst}
Daniel Russo.
\newblock Worst-case regret bounds for exploration via randomized value
  functions.
\newblock \emph{arXiv preprint arXiv:1906.02870}, 2019.

\bibitem[Russo and Van~Roy(2013)]{russo2013eluder}
Daniel Russo and Benjamin Van~Roy.
\newblock Eluder dimension and the sample complexity of optimistic exploration.
\newblock \emph{Advances in Neural Information Processing Systems}, 26, 2013.

\bibitem[Russo et~al.(2017)Russo, Van~Roy, Kazerouni, Osband, and
  Wen]{russo2017tutorial}
Daniel Russo, Benjamin Van~Roy, Abbas Kazerouni, Ian Osband, and Zheng Wen.
\newblock A tutorial on thompson sampling.
\newblock \emph{arXiv preprint arXiv:1707.02038}, 2017.

\bibitem[Sekhari et~al.(2021)Sekhari, Dann, Mohri, Mansour, and
  Sridharan]{sekhari2021agnostic}
Ayush Sekhari, Christoph Dann, Mehryar Mohri, Yishay Mansour, and Karthik
  Sridharan.
\newblock Agnostic reinforcement learning with low-rank mdps and rich
  observations.
\newblock \emph{Advances in Neural Information Processing Systems}, 34, 2021.

\bibitem[Simchowitz and Foster(2020)]{simchowitz2020naive}
Max Simchowitz and Dylan Foster.
\newblock Naive exploration is optimal for online lqr.
\newblock In \emph{International Conference on Machine Learning}, pages
  8937--8948. PMLR, 2020.

\bibitem[Sun et~al.(2019)Sun, Jiang, Krishnamurthy, Agarwal, and
  Langford]{sun2019model}
Wen Sun, Nan Jiang, Akshay Krishnamurthy, Alekh Agarwal, and John Langford.
\newblock Model-based rl in contextual decision processes: Pac bounds and
  exponential improvements over model-free approaches.
\newblock In \emph{Conference on learning theory}, pages 2898--2933. PMLR,
  2019.

\bibitem[Sutton et~al.(2009)Sutton, Maei, Precup, Bhatnagar, Silver,
  Szepesv{\'a}ri, and Wiewiora]{sutton2009fast}
Richard~S Sutton, Hamid~Reza Maei, Doina Precup, Shalabh Bhatnagar, David
  Silver, Csaba Szepesv{\'a}ri, and Eric Wiewiora.
\newblock Fast gradient-descent methods for temporal-difference learning with
  linear function approximation.
\newblock In \emph{Proceedings of the 26th Annual International Conference on
  Machine Learning}, pages 993--1000, 2009.

\bibitem[Sutton and Barto(1998)]{sutton1998reinforcement}
R.S. Sutton and A.G. Barto.
\newblock \emph{Reinforcement Learning: An Introduction}.
\newblock Adaptive Computation and Machine Learning series. MIT Press, 1998.
\newblock ISBN 9780262303842.
\newblock URL \url{https://books.google.com/books?id=U57uDwAAQBAJ}.

\bibitem[Swaminathan et~al.(2017)Swaminathan, Krishnamurthy, Agarwal, Dudik,
  Langford, Jose, and Zitouni]{swaminathan2017off}
Adith Swaminathan, Akshay Krishnamurthy, Alekh Agarwal, Miro Dudik, John
  Langford, Damien Jose, and Imed Zitouni.
\newblock Off-policy evaluation for slate recommendation.
\newblock \emph{Advances in Neural Information Processing Systems}, 30, 2017.

\bibitem[Thompson(1933)]{thompson1933likelihood}
William~R Thompson.
\newblock On the likelihood that one unknown probability exceeds another in
  view of the evidence of two samples.
\newblock \emph{Biometrika}, 25\penalty0 (3-4):\penalty0 285--294, 1933.

\bibitem[Uehara et~al.(2021)Uehara, Imaizumi, Jiang, Kallus, Sun, and
  Xie]{uehara2021finite}
Masatoshi Uehara, Masaaki Imaizumi, Nan Jiang, Nathan Kallus, Wen Sun, and
  Tengyang Xie.
\newblock Finite sample analysis of minimax offline reinforcement learning:
  Completeness, fast rates and first-order efficiency.
\newblock \emph{arXiv preprint arXiv:2102.02981}, 2021.

\bibitem[Wang et~al.(2020)Wang, Salakhutdinov, and Yang]{wang2020reinforcement}
Ruosong Wang, Russ~R Salakhutdinov, and Lin Yang.
\newblock Reinforcement learning with general value function approximation:
  Provably efficient approach via bounded eluder dimension.
\newblock \emph{Advances in Neural Information Processing Systems},
  33:\penalty0 6123--6135, 2020.

\bibitem[Yang and Wang(2020)]{yang2020reinforcement}
Lin Yang and Mengdi Wang.
\newblock Reinforcement learning in feature space: Matrix bandit, kernels, and
  regret bound.
\newblock In \emph{International Conference on Machine Learning}, pages
  10746--10756. PMLR, 2020.

\bibitem[Zanette et~al.(2020{\natexlab{a}})Zanette, Brandfonbrener, Brunskill,
  Pirotta, and Lazaric]{zanette2020frequentist}
Andrea Zanette, David Brandfonbrener, Emma Brunskill, Matteo Pirotta, and
  Alessandro Lazaric.
\newblock Frequentist regret bounds for randomized least-squares value
  iteration.
\newblock In \emph{International Conference on Artificial Intelligence and
  Statistics}, pages 1954--1964. PMLR, 2020{\natexlab{a}}.

\bibitem[Zanette et~al.(2020{\natexlab{b}})Zanette, Lazaric, Kochenderfer, and
  Brunskill]{zanette2020learning}
Andrea Zanette, Alessandro Lazaric, Mykel Kochenderfer, and Emma Brunskill.
\newblock Learning near optimal policies with low inherent bellman error.
\newblock In \emph{International Conference on Machine Learning}, pages
  10978--10989. PMLR, 2020{\natexlab{b}}.

\bibitem[Zhang(2021)]{Zhang2021FeelGoodTS}
Tong Zhang.
\newblock Feel-good thompson sampling for contextual bandits and reinforcement
  learning.
\newblock \emph{ArXiv}, abs/2110.00871, 2021.

\bibitem[Zhang et~al.(2022)Zhang, Song, Uehara, Wang, Sun, and
  Agarwal]{zhang2022efficient}
Xuezhou Zhang, Yuda Song, Masatoshi Uehara, Mengdi Wang, Wen Sun, and Alekh
  Agarwal.
\newblock Efficient reinforcement learning in block mdps: A model-free
  representation learning approach.
\newblock \emph{arXiv preprint arXiv:2202.00063}, 2022.

\bibitem[Zienkiewicz et~al.(2005)Zienkiewicz, Taylor, and
  Zhu]{zienkiewicz2005finite}
O.C. Zienkiewicz, R.L. Taylor, and J.Z. Zhu.
\newblock \emph{The Finite Element Method: Its Basis and Fundamentals}.
\newblock Elsevier Science, 2005.

\end{thebibliography}

\appendix

\section{Related Work}
\label{sec:related}

\mypar{Settings for sample-efficient RL.} There is substantial recent literature~\citep{jiang2017contextual,sun2019model,jin2020provably,yang2020reinforcement,du2019provably,du2021bilinear,jin2021bellman,foster2021statistical} on structural conditions that enable sample-efficient RL. Of these, the Bellman rank framework and its recent generalization in the Bilinear Classes model remain the broadest known frameworks for which provably sample-efficient methods are known. While Bellman rank itself makes no assumptions on the complexity of the action space (\citet{jiang2017contextual} show small Bellman rank for LQRs), the algorithm \olive developed for this setting crucially relies on importance sampling over a discrete action set. Ideas from these works have further been developed to a representation learning setting, where the transition dynamics of the MDP are linear in some \emph{unknown features} and the agent learns this map, given a class of candidate features~\citep{agarwal2020flambe,uehara2021finite,modi2021model}, which captures rich non-linear function approximation in the original state and action. \citet{jin2021bellman} developed the Bellman-Eluder dimension notion to better capture infinite action sets using a notion they call $Q$-type Bellman rank, while the original version of~\citet{jiang2017contextual} is termed the $V$-type Bellman rank. Note, however, that the GOLF algorithm of~\citet{jin2021bellman} for problems with a small $Q$-type Bellman rank, which scales to infinite action spaces, cannot be used for feature learning and does not capture all contextual bandit problems (see Section~\ref{sec:Qtype-lb} for a lower bound), showing the limitations of this approach in terms of the non-linearity it permits. For $V$-type Bellman-Eluder dimension, \citet{jin2021bellman} also rely on uniform exploration over actions similar to \olive. The techniques developed here do scale to feature learning, capture contextual bandits fully and apply to infinite action scenarios satisfying the linear embedding assumption. At the same time, our assumptions do not capture all problems with a small $Q$-type Bellman rank, so there are problems which GOLF can handle which we do not. That said, perhaps the most prominent example for GOLF in the infinite action setting is that of linear MDPs, where the linear embedding assumption made here holds.

\mypar{Continuous control.} Large action spaces are the standard framing in continuous control, where the action is typically a vector in $\R^d$ for some control input dimension $d$. While there has been a number of recent results at the intersection of learning and control~\citep[see e.g.][]{dean2020sample,mania2019certainty,agarwal2019online,simchowitz2020naive}, a large body of work typically focuses on highly structured settings such as the Linear Quadratic Regulator (LQR), where online exploration is is straightforward due to the presence of a Gaussian noise in the dynamics. More recent results~\citep{kakade2020information,mania2020active} do combine online control and exploration, they typically focus on model-based settings and still rely on access to good features. We note that~\citet{mhammedi2020learning} carry out feature learning for continuous control, but their setting does not have a small Bellman rank and hence is not admissible in our conditions either.

\mypar{Posterior sampling.} Posterior sampling methods for RL, motivated by Thompson sampling~\cite{thompson1933likelihood}, have been extensively developed and analyzed in terms of their expected regret under a Bayesian prior by many authors~\citep[see e.g.][]{osband2013more,russo2017tutorial,osband2016generalization} and are often popular as they offer a simple implementation heuristic through approximation by ensembles trained on random subsets of data. Worst-case analysis of TS in RL settings has also been done for both tabular~\citep{russo2019worst,agrawal2017posterior} and linear~\citep{zanette2020frequentist} settings. Our work is most closely related to the recent Feel-Good Thompson Sampling strategy proposed and analyzed in~\citep{Zhang2021FeelGoodTS}, primarily for bandits but also for RL problems with deterministic dynamics. They study problems with a similiar linear embeddability assumption, but the absence of any further structure like a small Bellman rank precludes their techniques for application to general stochastic dynamics. We also observe that FGTS retains the significant optimistic component of other exploration techniques like LSVI-UCB~\citep{jin2020provably} and \olive~\citep{jiang2017contextual}, which partly explains its success in worst-case settings. On the other hand, the approach has the remarkable property of working for both linear bandits and non-linear bandits with finite action sets with an identical algorithm and analysis, and we extend that benefit to RL in this work.

\mypar{Minimax objectives in RL.} FGTS algorithm uses a likelihood term for Thompson Sampling based on the TD error, which is well-known to have a bias in estimating the Bellman error~\citep{antos2008learning,sutton1998reinforcement} for stochastic dynamics. The usual technique of removing the residual variance from~\citet{antos2008learning} creates a minimax objective, which we also use in this paper. Other minimax formulations~\citep{dai2018sbeed} to remove this bias are also possible in general, but we find that the one from~\citet{antos2008learning} is the most natural under our structural assumptions. We also note that approach of keeping two timescales~\citep{borkar2009stochastic} used here has been used previously in offline RL for TD learning methods~\citep{sutton2009fast}, but its online analysis in TS appears novel, and different from the analysis in \cite{DMZZ2021-neurips}, which cannot handle general nonlinear feature learning considered here. 

\section{Lower bound for $Q$-type Bellman rank}
\label{sec:Qtype-lb}

We now instantiate a contextual bandit problem where the realizability assumption holds, but the $Q$-type Bellman rank grows linearly in the number of contexts. The construction is due to \citet{dann2018personal}, but has not appeared in the literature. Note that the $V$-type Bellman rank of~\citet{jiang2017contextual}, which we further generalize in this work, is always $1$ for a contextual bandit problem.
The lower bound is demonstrated using a typical hard instance for contextual bandit problems. Let us consider a context distribution which is uniform on $[N]$, where we have $N$ unique contexts. We have two actions $\{a_1, a_2\}$. We also have $|\cF| = N+1$ with the following structure.
\begin{align*}
    f^\star(x,a_1) &= f_{N+1}(x,a_1) = 0, \quad \mbox{and} f^\star(x,a_2) = f_{N+1}(x,a_2) = 0.5.
\end{align*}
For $i < N+1$, we have $f_i(x,a) = f^\star(x,a)$ when $x \ne i$, and $f_i(x,a_1) = 1$, $f_i(x,a_2) = 0.5$ so that $f_i$ makes incorrect predictions on the context $i$ for action $a_1$. Notice that the design of $\cF$ also ensures that for each context $i$, there is a policy $\pi_\star$ (greedy wrt $f^\star$) which picks the action $a_2$, while another policy $\pi_i$ (greedy wrt $f_i$) which picks action $a_1$. Since this is a problem with horizon $H=1$, the Bellman error is simply equal to $f(x,a) - f^\star(x,a)$ for any $x,a$. Let $\BE^Q(f,\pi) = \rE_{x,a\sim \pi} \BE(f,f,x,a)$ be the $Q$-type Bellman error. Then, we observe that for $1 \leq i,j\leq N$:
\begin{align*}
    \BE^Q(f_i,\pi_j) &= \frac{1}{N}\sum_{i=1}^N\rE_{a\sim \pi_j} \left[f(x,a) - f^\star(x,a)\right] = \frac{\mathbf{1}(j = i)}{N},
\end{align*}
where $\mathbf{1}(\cdot)$ is an indicator function. Hence, we see that the Bellman error matrix contains an identity submatrix of size $N$, so that its rank is at least $N$.

\section{Connection with sparse RL}
\label{sec:sparse-rl}

In this section, we formalize the relationship between representation learning and RL with sparsity. Concretely, let us consider two formulations.

\begin{equation}
    P(x'\mid x,a) = \phi^{lr}(x,a)^\top\mu^{lr}(x'),\quad \mbox{where}\quad (\phi^{lr},\mu^{lr})\in\Omega^{lr},\tag{$P1$}
    \label{eq:replearn}
\end{equation}
is the standard model-based representation learning formulation for RL~\citep{agarwal2020flambe,uehara2021finite}. Here we consider a slightly generalized setup where the feature maps $(\phi,\mu)$ are available as pairs from a joint set $\Omega^{lr}$ instead of separate classes for $\phi$ and $\mu$, which reduces to the more typical framing with separate classes for $\phi$ and $\mu$ when we set $\Omega^{lr}$ to be the cartesian product of their respective sets. However, all existing representation learning algorithms can handle this setup for a general $\Omega^{lr}$ without any extra difficulty. Let us denote $d^{lr} = \text{dim}(\phi^{lr}$, where we use the superscript $lr$ to denote the low-rank problem $P1$. The second formulation is the sparse linear MDP setup of~\citet{hao2021online}.

\begin{equation}
    P(x'\mid x,a) = \sum_{i\in \idxset} \phi^s(x,a)_i^\top\mu^s(x')_i, \quad \mbox{where $|\idxset| = k \ll d^s = \text{dim}(\phi^s)$,} \tag{$P2$} 
    \label{eq:splinear}
\end{equation}
with the feature maps $\phi^s,\mu^s$ considered known in $P2$, but the index set $\idxset$ is unknown. We now show that the two problems are equivalent when $d^{lr}\cdot|\Omega^{lr}| = d^s$ and $k = d^{lr}$, in that any solution to $P1$ can solve $P2$ at no additional sample cost, and vice versa. 

In one direction, We define the concatenated feature map for all $x,a$ and $x'$:
\begin{align*}
    \phi^{all}(x,a) &= (\phi_1(x,a),\phi_2(x,a),\ldots,\phi_N(x,a)), \quad \mbox{where $N = |\Omega^{lr}|$, and}\\\mu^{all}(x') &= (\mu_1(x'),\mu_2(x'),\ldots,\mu_N(x')),
\end{align*}
where $\phi_i,\mu_i$ refer to the $i_{th}$ feature maps in $\Omega^{lr}$. Clearly, the assumption $\phi^{lr},\mu^{lr}\in\Omega$ guarantees that $P(x'\mid x,a) = \phi^{all}(x,a)^\top \mu^{all}(x')$ with $\text{dim}(\phi^{all}) = d^{lr}|\Omega^{lr}| = d^s$, by construction. However, the transitions are also sparse in the features $(\phi^{all},\mu^{all})$, in that we can choose the $d^{lr}$ coordinates of $\phi^{all}, \mu^{all}$ which correspond to the index of $(\phi^{lr},\mu^{lr})$ in $\Omega^{lr}$. Consequently, any solution to $P2$ for $k = d^{lr}$ which uses samples that are $\order(\text{poly}(k\log d^s)$ can be used to solve $P1$ with a sample complexity of $\order(\text{poly}(d^{lr}\log (d^{lr}|\Omega^{lr}|))$, which is considered the standard goal of the representation learning problem $P1$.

In the other direction, let us say we have a solution to $P1$ with a sample complexity that is $\order(\text{poly}(d^{lr}\log (|\Omega^{lr}|))$. Then we given a pair of high-dimensional feature maps $\phi^s,\mu^s$, we define a class $\Omega^{lr}$ as follows:
\begin{align*}
    \Omega^{lr} = \{(\phi,\mu)~:~\phi(x,a) = \phi^s(x,a)_{\otheridx}, \mu(x') = \mu^s(x')_{\otheridx}, \quad \mbox{where $\otheridx \subseteq \{1,\ldots,d^s\}$ with $|\otheridx| = k$}\}.
\end{align*}

In words, we add features corresponding to every subset of size $k$ from the original $d^s$ features as a candidate representation to $\Omega^{lr}$. Then $|\Omega^{lr}| = {d^s\choose k} = \order((d^s)^k)$ and each feature map in $\Omega^{lr}$ has $d^{lr} = k$. Clearly the sparse linear MDP assumption in the definition of $P2$~\eqref{eq:splinear} implies that $\Omega^{lr}$ contains a pair $(\phi,\mu)$ under which the MDP is linear. Furthermore, our method for representation learning applied to the sparse linear MDP has a sample complexity of $\order(\text{poly}(k\log d^s))$.

The above reduction show that any obstacle to sparse linear MDP learning also results in a lower bound for representation learning. In particular, the construction of~\citet{hao2021online} precludes learning sparse linear MDPs where the action set has a cardinality $\order(d^s)$, so that we cannot expect to carry out representation learning in arbitrarily large action spaces without further assumptions.

\section{Examples satisfying Assumption~\ref{assumption:embedding}}
\label{sec:examples}

\begin{example}[Basis expansions in action space]
    For continuous action problems, \citet{asadi2021deep} study the class of deep radial basis value functions, where any $f\in\cF$ takes the form
    \begin{equation*}
        f(x,a;\theta) = \frac{\sum_{i=1}^N \exp(-\beta\|a - a_i(x;\theta)\|)v_i(x;\theta)}{\sum_{i=1}^N \exp(-\beta\|a - a_i(x;\theta)\|)},
    \end{equation*}
    where $a_i(x;\theta)$ are a (state and $\theta$-dependent) basis, while $v_i(x;\theta)$ are some reference values at these points. If we consider the case where the $a_i(x;\theta)$ are only dependent on the state $x$, but fixed across $\theta$, then this definition satisfies Assumption~\ref{assumption:embedding} with $w(f,x) = v_i(x;\theta)$ and $\phi(x,a)_i = \nicefrac{\exp(-\beta\|a - a_i(x;\theta)\|)}{\sum_{i=1}^N \exp(-\beta\|a - a_i(x;\theta)\|)}$. More generally, whenever there is a basis set of actions such that $f(x,a) = \sum_{i=1}^N \alpha_i(x;a) g_f(x,a_i)$ for some functions $\alpha_i$ and $g_f$, then the assumption holds. Such a basis can be obtained, for instance, by standard finite element approximation techniques such as triangulation in an appropriate metric or other finite element techniques in the action space~\citep{zienkiewicz2005finite}, as long as $f$ is sufficiently smooth in $a$.
    \label{ex:basis}
\end{example}

\begin{example}[Block MDP with latent state, action features]
    Several authors study the Block MDP model~\citep{du2019provably,misra2019kinematic}, which is a special case of MDPs with a small Bellman rank. In a block MDP, each observation $x\in\cX$ is emitted a \emph{unique} latent state $z\in\cZ$, so that there is an unknown one-to-one mapping $g^\star~:~\cX\to\cZ$ which maps an observed state to the latent state which generated it. Let $\psi(x^{h},a^h) \in \R^{|\cZ|\times|\cA|}$ be an indicator vector $e_{(g^\star(x^h),a^h)}$ with a one corresponding to the entry $(g^\star(x^h), a^h)$ and $0$ everywhere else. Then the MDP dynamics are linear in these features~\citep[see e.g.][]{zhang2022efficient}, and \olive solves this model using $\text{poly}(|\cZ|,|\cA|)$ samples with no explicit dependence on $|\cX|$ when the actions are finite, with computationally efficient approaches provided in later works~\citep{du2019provably,misra2019kinematic}. In the setup of this work, let us consider a more expressive model where $|\cA|$ can be infinite, but we are given a feature map $\zeta(z,a)\in\R^d$ such that $P(z^{h+1}|z^h,a^h)$ and the reward function are both linear in $\phi$, that is:
    \[
        r(x,a) = w_r^\top\zeta(g^\star(x),a),~\mbox{and}~ P(z^{h+1}=z_i|z^h,a^h) = w_P(i)^\top\zeta(z^h,a^h),
    \]
    for some $w_r,w_P(1),\ldots w_P(|\cZ|)\in\R^d$. Then it is easily shown that for any function $v~:~\cX\to\R$, we have $\rE[r(x^h,a^h) + v(x^{h+1}) \mid x^h,a^h] = w_v^\top\zeta(x^h,a^h)$ for some $w_v\in\R^d$. Consequently, Assumption~\ref{assumption:embedding} is satisfied when $f(x^h,a^h)$ is linearly embeddable in some features $\zeta'(x^h,a^h)$ and we set $\phi(x^h,a^h) = (\zeta'(x^h,a^h),\zeta(x^h,a^h))$ to be the concatenation of the two feature sets. Note that we require $\cF$ to satisfy completeness and realizability assumptions, so a natural choice is to pick $\zeta' = \zeta$ when $\zeta$ is known. Then our function class takes the form $\{f(x^h,a^h) = w(f,x^h)^\top \zeta(x^h,a^h)\}$ for some choice of the mapping $w$. For instance, we may pick $w(f,x) = w_f^TM(x)$, where $w_f\in \R^{|\cZ|}$ is a vector and $M(x)\in R^{|\cZ|\times d}$ is a matrix which encodes $g^\star(x)$ as a non-zero row in the matrix. As before, this generalizes the finite action scenario where $\zeta$ can be chosen to be indicator features of actions. More generally, this corresponds to problems where we have a reasonable embedding for actions in each latent state. This might happen, for instance, in control problems where given the agent's physical state, reasonable control policies can be designed by using action embeddings that are either random, or learned apriori~\citep[see e.g.][]{rajeswaran2017towards,kakade2020information}, but the physical position might not be available when the observation is based on vision or other sensory measurements.
\end{example}

\section{Proof of Proposition~\ref{prop:spectral}}
\label{sec:spectral}

We start by reviewing the bound on $\br(\epsilon)$ for the finite dimensional case. In this case, let $\{\lambda_i\}_{i=1}^d$ be the eigenvalues of $\Sigma^h(p,x^1)$ in decreasing order, for some distribution $p\in\Delta(\cF)$ and $x^1\in\cX$, where we recall that $\lambda_d \geq 0$ since $\Sigma^h(p,x^1)$ is a covariance matrix. Then we have
\begin{align*}
    \trace((\Sigma^h(p,x^1) + \lambda I)^{-1}\Sigma^h(p,x^1)) &= \sum_{i=1}^d \frac{\lambda_i}{\lambda + \lambda_i} \leq \sum_{i=1}^d 1 \leq d.
\end{align*}

\mypar{Proof for geometric decay case.} This follows effectively by reducing to the finite dimensional case. For any positive integer $n$, we have

\begin{align*}
    \trace((\Sigma^h(p,x^1) + \lambda I)^{-1}\Sigma^h(p,x^1)) &\leq n + \frac{\alpha^n}{\lambda (1-\alpha)}.
\end{align*}

Choosing $n = n_0$ such that $\frac{\alpha^{n_0}}{(1-\alpha)} = \epsilon^2/2$, we see that 

\begin{align*}
    \lambda K^h(\lambda) &\leq n\lambda + \frac{\epsilon^2}{2},
\end{align*}
so that we can set $\lambda = \epsilon^2/2n_0$. With these settings, we have 
\[
K^h(\lambda) \leq n_0 + \frac{2n_0\alpha^{n_0}}{\epsilon^2 (1-\alpha)} = 2n_0 = 2\frac{\log\tfrac{2}{\epsilon^2(1-\alpha)}}{\log\tfrac{1}{\alpha}},
\]
where the last equality follows from our setting of $n_0$.

\mypar{Proof for polynomial decay case.} In this case, the assumption on trace guarantees that 

\[  
    \sum_i \lambda_i^q \leq R_q.
\]

Now since $q \in (0,1)$, we have

\begin{align*}
    \trace((\Sigma^h(p,x^1) + \lambda I)^{-1}\Sigma^h(p,x^1)) &\leq \sum_i \frac{\lambda_i}{\lambda + \lambda_i}\\
    &\leq \sum_i \left(\frac{\lambda_i}{\lambda+\lambda_i}\right)^q \tag{$x^q \geq x$ for $x < 1$ and $q\in(0,1)$}\\
    &\leq \sum_i \frac{\lambda_i^q}{\lambda^q} \leq \frac{R_q^q}{\lambda^q}.
\end{align*}

With this, we have $\lambda K^h(\lambda) \leq \lambda^{1-q} R_q^q$, so that we choose 
\[
    \lambda_0 = \left(\epsilon^2R_q^{-q}\right)^{1/(1-q)}.
\]
With this choice, we get 
\begin{align*}
    K^h(\lambda_0) &\leq \frac{R_q^q}{\lambda_0^q} = \left(\epsilon^{-2}R_q\right)^{q/(1-q)}.
\end{align*}

\section{Proof of Corollary~\ref{cor:poly}}
\label{sec:cor}

\begin{proof}[Proof of Corollary~\ref{cor:poly}] 
    The calculations for setting the parameters in this case are a little more tedious. We work under the assumption $\mu > 1$, which is subsequently satisfied. Then optimizing the leading order terms under this assumption yields $\epsilon_1 = \left(\frac{d_q}{\mu B}\right)^{(1-q)/(1+q)}$ and $\epsilon_2 = \left(\frac{\lambda H\mu d}{\eta B}\right)^{(1-q)/(1+q)}$. Now the optimal setting of $\mu = (\eta/(\lambda H))^{(1-q)/(3-q)}$ under our assumption of $\mu \geq 1$. We now set $\eta$ as in Theorem~\ref{thm:main} and optimize to get $\lambda = \order\left(T^{-(3-q)(1+q)/4}
    (\ln N)^{(1-q^2+2q)/2}c_q^{-(3-q)(1+q)/4}\right)$, where $c_q = d_q^{(1-q)/(1+q)} H^{4/((3-q)(1+q))} B^{2q/(1+q)}$. Substituting this back into Theorem~\ref{thm:main} completes the proof.
\end{proof}

\section{Proofs of decoupling results}
\label{sec:decoupling}

In this section, we prove Lemma~\ref{lem:decouple}. This will be done across two smaller lemmas which provide one level of decoupling using the Bellman rank and another using the linear embedding. The first lemma's proof technique resembles the recently used "one-step back" inequalities in low-rank MDP literature~\citep{agarwal2020flambe,uehara2021finite}. Throughout this section, we use $u^\top v$ to denote $\langle u,v\rangle$ even for infinite dimensional vectors with a slight abuse of notation to improve readability.

\begin{lemma}[Bellman Rank Decoupling]
Consider  any distribution $p$ over $\cF$. The following inequality holds for any $\epsilon, \mu_1 > 0$.
\begin{align*}
&\rE_{f \sim p}  \rE_{(x^h,a^h) \sim \pi_f|x^1} 
|\BE^h(f,f;x^h,a^h)|\\
\leq& \sqrt{\br^{h-1}(\epsilon) \rE_{f' \sim p}\rE_{x^h \sim \pi_{f'}|x^1} [\rE_{f \sim p} 
\BE^h(f,f;x^h,\pi_{f}(x^h))^2]}
 + \epsilon B_1 \\
  \leq&
\mu_1        \rE_{f' \sim p}\rE_{x^h \sim \pi_{f'}|x^1} \rE_{f \sim p} 
[\BE^h(f,f;x^h,\pi_{f}(x^h))^2] +       \mu_1^{-1} \br^{h-1}(\epsilon)   
 + \epsilon B_1.
\end{align*}
\label{lem:decouple-br}
\end{lemma}
\begin{proof}
We have
\begin{align*}
  & \rE_{f \sim p} \rE_{(x^h,a^h) \sim \pi_f|x^1} 
| \BE^h(f,f;x^h,a^h)| \\
=& \rE_{f \sim p} \rE_{x^{h} \sim \pi_f|x^1} 
  \rE_{x^{h}|x^{h-1},a^{h-1}} 
|\BE^h(f;x^h,\pi_f(x^h))|\\
=& \rE_{f \sim p} |u^h(f,x^1)^\top 
  \psi^{h-1}(f,x^1)| \tag{Assumption~\ref{ass:bellman}}\\
\leq& \sqrt{K^{h-1}(\lambda) \rE_{f \sim p} u^h(f,x^1)^\top (\Sigma^{h-1}(p,x^1)+ \lambda I)u^h(f,x^1)} \tag{$\rE[u^Tv] \leq \sqrt{\rE[\|u\|_M^2] \; \rE[ \|v\|_{M^{-1}}^2]}$ for psd $M$ by Cauchy-Schwarz and Definition~\ref{def:br}}\\
\leq& \sqrt{K^{h-1}(\lambda) \rE_{f \sim p} u^h(f,x^1)^\top (\rE_{f' \sim p} \psi^{h-1}(f',x^{1})\psi^{h-1}(f',x^{1})^\top )u^h(f,x^1)}\\
&+ \sqrt{\lambda K^{h-1}(\lambda) \rE_{f \sim p} \|u^h(f,x^1)\|_2^2} \tag{$\sqrt{a+b} \leq \sqrt{a} + \sqrt{b}$ for $a,b \geq 0$} \\
=& \sqrt{K^{h-1}(\lambda)  \rE_{f \sim p}\rE_{f' \sim p} (u^h(f,x^{1})^\top \psi^{h-1}(f',x^{1}))^2} 
 + \sqrt{\lambda K^{h-1}(\lambda,p) \rE_{f \sim p} \|u^h(f,x^1)\|_2^2} \\
 \stackrel{(a)}{=}& \sqrt{K^{h-1}(\lambda)  \rE_{f \sim p}\rE_{f' \sim p} (\rE_{(x^{h}) \sim \pi_{f'}|x^1} \BE^h(f,f;x^h,\pi_f(x^h)) )^2} 
 + \sqrt{\lambda K^{h-1}(\lambda,p) \rE_{f \sim p} \|u^h(f,x^1)\|_2^2} \\
\stackrel{(b)}{\leq}&
\sqrt{K^{h-1}(\lambda) \rE_{f \sim p} \rE_{f' \sim p}\rE_{(x^{h}) \sim \pi_{f'}|x^1}
\left( \BE^h(f,f;x^h,\pi_f(x^h)))^2 \right)} 
+ \sqrt{\lambda K^{h-1}(\lambda) \rE_{f \sim p} \|u^h(x^1,f)\|_2^2} .
\end{align*}
Here $(a)$ holds by using Assumption~\ref{ass:bellman} once more to rewrite the inner product as a Bellman error, while $(b)$ is a consequence of Jensen's inequality. Now we recall that 

\[
   \|u^h(f,x^1)\|_2 \leq B_1 .
\]
By taking the largest $\lambda$ so that $\lambda \sup_p K^{h-1}(\lambda,p) \leq
\epsilon^2$, we obtain the desired bound. 
\end{proof}

The next decoupling lemma separates the sampling of the action $a^h$ from the function $f$ whose Bellman error is being evaluated by using Assumption~\ref{ass:complete}, which is crucial for the analysis as explained in Section~\ref{sec:proof-main}. Note that in this particular derivation, we reduce squared Bellman error to the square root of a decoupled squared Bellman error, which loses rate. It may be possible to improve this reduction by a more careful analysis in future work. 
\begin{lemma}[Linear Embedding Decoupling]
We have for all $x^h \in \cX^h$ and probability measures $p$ on $\cF$ and $\mu_2,\epsilon_2 > 0$:
\begin{align*}
 \rE_{f \sim p} [
\BE^h(f,f;x^h,\pi_{f}(x^h))^2]
\leq &2\,\sqrt{\dc^h(\epsilon) \rE_{f \sim p,f'\sim p}\rE_{a^h\sim\pi_{f'}}\left[ \BE^h(f,f;x^h,a^h)^{2}\right]} + 2 \epsilon B_2\\
\leq & 2\mu_2\rE_{f \sim p,f'\sim p}\rE_{a^h\sim\pi_{f'}} \BE^h(f,f;x^h,a^h)^{2} +2\mu_2^{-1} \dc^h(\epsilon)     + 2\epsilon B_2.
\end{align*}
\label{lem:decouple-dc}
\end{lemma}
\begin{proof}
We have
\begin{align*}
&\rE_{f\sim p} \; |\BE^h(f,f;x^h,\pi_{f}(x^h))| = \rE_{f\sim p} \; |w^h(f,x^h)^\top \phi^h(x^h,\pi_{f}(x^h))| \tag{Assumption~\ref{assumption:embedding}}\\
&\leq \left[\rE_{f \sim p} w^h(f,x^h)^\top (\tilde{\Sigma}^h+ \tilde{\lambda} I) w^h(f,x^h)\right]^{1/2}
\left[\rE_{f \sim p} \phi^h(x^h,\pi_{f}(x^h))^\top 
(\tilde{\Sigma}^h+ \tilde{\lambda} I)^{-1}
\phi^h(x^h,\pi_{f}(x^h))\right]^{1/2} \tag{Cauchy-Schwarz} ,
\end{align*}
where $\tilde{\Sigma}^h$ 
is short for $\tilde{\Sigma}^h(p,x^h)$.
 Therefore we have
 \begin{align*}
 &
 \rE_{f \sim p} [\BE^h(f,f;x^h,\pi_{f}(x^h))]^2 \\
 \leq& 2\,\rE_{f \sim p} [|\BE^h(f,f;x^h,\pi_{f}(x^h))|] \tag{$|\BE^h(f,f;x^h,\pi_{f}(x^h)) \leq 2|$ since $r^h, f(x,a)\in[0,1]$}\\
\leq&  
2\,\sqrt{\left[ \rE_{f \sim p} (w^h(f,x^h)^\top (\tilde{\Sigma}^h+ \tilde{\lambda} I) w^h(f,x^h))\right]
\left[ \rE_{f \sim p} (\phi^h(x^h,\pi_{f}(x^h))^\top 
(\tilde{\Sigma}^h+ \tilde{\lambda} I)^{-1}
\phi^h(x^h,\pi_{f}(x^h)))\right]} \\
\stackrel{(a)}{\leq} & 
2\,\sqrt{\left[ \rE_{f \sim p,f'\sim p} \rE_{a^h\sim\pi_{f'}} (w^h(f,x^h)^\top (\phi^h(x^h,a^h)\phi^h(x^h,a^h)^\top+ \tilde{\lambda} I) w^h(f,x^h))\right]
\tilde{K}^h(\tilde{\lambda})} \\
\stackrel{(b)}{=} & 
2\,\sqrt{\left[\rE_{f \sim p,f'\sim p} \rE_{a^h\sim\pi_{f'}} (\BE^h(f,f;x^h,a^h)^2+ \tilde{\lambda}  \|w^h(f,x^h)\|_2^2)\right]
\tilde{K}^h(\tilde{\lambda})} \\
\leq &
 2\,\sqrt{\left[ \rE_{f \sim p,f'\sim p} \rE_{a^h\sim\pi_{f'}} (
\BE^h(f,f;x^h,a^h)^2 )\right] \tilde{K}_{q_2}^h(\tilde{\lambda})}
       + 2\sqrt{\tilde{\lambda}  \tilde{K}^h(\tilde{\lambda})
      \rE_{f \sim p} \|w^h(f,x^h)\|_2^2} \tag{$\sqrt{a+b} \leq \sqrt{a} + \sqrt{b}$} .
 \end{align*}
 Here $(a)$ holds due to the definition of $\tilde{K}^h$ (Definition~\ref{def:dc}).
 Note that
 \[
 \|w^h(f,x^h)\|_2 \leq B_2 .
 \]
 By taking the largest $\tilde{\lambda}$ so that $\tilde{\lambda} \sup_p \tilde{K}^{h-1}(\tilde{\lambda},p) \leq
\epsilon^2$, we obtain the first desired bound. The second bound follows by Cauchy-Schwarz inequality.
\end{proof}

\mypar{Proof of Proposition~\ref{lem:decouple}.} Using the two lemmas above, it is easy to establish Proposition~\ref{lem:decouple}.

\begin{proof}[Proof of Proposition~\ref{lem:decouple}]
Since $x_t^1$ is randomly drawn from $\cD$, it is independent of $f_t$. Therefore we have
  \begin{align*}
   &\lambda \rE \; \regret(f_t,x_t^1)+\lambda \rE \rE_{f|S_{t-1}} \;\Delta f^1(x_t^1)
   = \lambda \rE \; \regret(f_t,x_t^1)+\lambda \rE \;\Delta f_t^1(x_t^1)
   \\
    =& \lambda \sum_{h=1}^H \rE \; \rE_{(x,a)\sim \pi_{f_t}|x_t^h}
     \BE^{h}(f_t,f_t,x^{h},a^h)
     = \lambda H\rE \; 
     \BE^{h_t}(f_t,f_t,x_t^{h_t},a_t^{h_t})\\
    \leq&   \lambda\left[ \epsilon_1 H B_1 +  \mu_1^{-1} \br(\epsilon_1) +
          H \mu_1 \rE \;  \rE_{f|S_{t-1}} \BE^{h_t}(f,f,x_t^{h_t},\pi_{f}(x_t^{h_t}))^2\right]\\
    \leq & \lambda\left[2\mu_1 \mu_2 H \rE\;
 \rE_{f|S_{t-1}} \BE^{h_t}(f,f;x_t^{h_t},a_t^{h_t})^{2 } + 
2\mu_1 \mu_2^{-1} \dc(\epsilon_2)
           + 2\epsilon_2 \mu_1 H B_2 + \epsilon_1 H B_1 + \mu_1^{-1} \br(\epsilon_1)\right] .
  \end{align*}
  The first equality used the independence of $x_t^1$ and $f_t$. The second equality used
  Proposition~\ref{prop:value-decomposition}.  The third equality used the fact that $h_t$ is uniformly drawn from $[H]$. The first inequality  used Lemma~\ref{lem:decouple-br}. The second inequality used Lemma~\ref{lem:decouple-dc}.
  Note that the last two steps crucially use that the function $f_t$ used to draw $x_t^{h_t}$, $f'_t$ to draw $a_t^{h_t}$ and $f\sim p$ being evaluated are all mutually independent. Taking $\mu_1=\mu$ and $\mu_2= \eta/(4\lambda\mu_1 H)$, we obtain the desired
  result. 
  \end{proof}

\section{Convergence of the online learning in \alg}
\label{sec:online}

In this section, we study the convergence of the online learning updates for the $f$ and $g$ functions. To do so, it is helpful to define some additional notation. 

Define $\sigma=2$, and let
\[
\hat{\epsilon}_t^h(f)= \cT^h f(x_t^h,a_t^h) - [r_t^h + f^{h+1}(x_t^{h+1})]
\]
be the noise in the Bellman residuals. We observe that $|\hat{\epsilon}_t^h(f)| \leq \sigma$ for all $t\in[T]$, $h\in[H]$ and $f\in\cF$ under our normalization assumptions. For convenience, we use $\sigma$ to denote this upper bound on the $\hat{\epsilon}_t^h(f)$ to avoid carrying constants. We also have the following observations about the noise.

For each $h \geq 1$, we also define:
\begin{align}
\hatd_t^h(g,f) &= \hat{\Delta}_t^h(g,f)^2 - \hat{\epsilon}_t^h(f)^2,\nonumber\\
\hatd_{t}^h(f) &= \rE_{g \sim p(g|f,S_{t-1})}\hatd_t^h(g,f),\quad\hspace*{\fill}\quad\mbox{and}\nonumber\\
\hatc_{t}^h(f) &= -\frac1{\gamma} \ln \rE_{g \sim p(g|f,S_{t-1})} \exp(-\gamma\hatd_t^h(g,f)).
\label{eq:delta-defs}
\end{align}
Here $\hatd_t^h(g,f)$ measures how well $g$ captures the Bellman residual of $f$ and $\hatd_t^h(f)$ measures the quality of our posterior distribution over $g$ in doing so. Finally $\hatc_t^h(f)$ is a log-partition function for the posterior. We further define the log-partition function for the posterior over $f$:
\[
Z(S_t) = - \ln \rE_{f \sim {p}_0} \exp\left(\sum_{s=1}^t \lambda \Delta f(x_s^1)-\eta \sum_{s=1}^t  [\hatd_s^{h_s}(f,f) - \hatc_{s}^{h_s}(f)]\right) ,
\]
where
\[
\Delta f(x^1)= f(x^1) - Q_\star^1(x^1) .
\]

We also introduce the following definition
\[
\hatc_t = \hatc_t^{h_t} = -\frac1\eta \ln \rE_{f|S_{t-1}} \exp\left(\lambda \Delta f(x_t^1)-\eta [\hatd_t^{h_t}(f,f) - \hatc_{t}^{h_t}(f)]\right) ,
\]
which is the normalization factor of $p(f|S_t)/p(f|S_{t-1})$.

Let $\psi(z)=(e^z-z-1)/z^2$. It is known that $\psi(z)$ is an increasing function of $z$. We organize the rest of this section as follows. We begin with some basic properties of Bellman residuals, then analyze the convergence of outer and inner updates respectively.

\subsection{Properties of Bellman residuals}

\begin{lemma}
For any $f$ that may depend on $S_{t-1}$:
\[
\rE_{x_t^{h+1},r_t^h|x_t^h, a_t^h,h_t=h} \hat{\epsilon}_t^h(f) = 0,
\]
and for any constant $b_t$ independent of $\hat{\epsilon}_t^h(f)$:
\[
\rE_{x_t^{h+1},r_t^h|x_t^h, a_t^h,h_t=h} \exp(b_t \hat{\epsilon}_t^h(f)) \leq 
\exp ( b_t^2 \sigma^2/2) .
\]
\label{lem:noise}
\end{lemma}

The first equality follows since the conditional expectation only acts over the MDP rewards and dynamics, which are conditionally independent of any $f$ that even depends on $S_{t-1}$. The second bound is a consequence of the sub-Gaussian bound for bounded random variables in the proof of Hoeffding's inequality.

\begin{lemma}
Given any $g,f$. We have
\[
-\sigma^2 \leq \hatd_t^h(g,f) = \BE^h(g,f;x_t^h,a_t^h)^2
+ 2 \hat{\epsilon}_t^h(f)\BE^h(g,f;x_t^h,a_t^h)\leq 1 +2\sigma  .
\]
Therefore
\[
\max\big(|\hatd_t^h(f,f)-\hatc_t^h(f)|, |\hatd_t^h(f,f)-\hatd_t^h(f)|\big) \leq (1+\sigma)^2 .
\]
Moreover for any $t\in[T]$, $h\in[H]$, $x\in\cX$ and $a\in\cA$, we have 
\[
\rE [\hatd_t^{h}(g,f) \mid x_t^h, a_t^h] =  \BE^{h}(g,f;x_t^h,a_t^h)^2 , \quad \rE [\hatd_t^{h}(g,f)^2 \mid x_t^h, a_t^h] \leq  (1 +4\sigma^2) \BE^{h}(g,f;x_t^h,a_t^h)^2 ,
\]
and for any $c > 0$,
\begin{align*}
\rE[\exp(c\hatd_t^h(g,f) \mid x_t^h, a_t^h] &\leq \exp(c(1+2c\sigma^2)\BE(g,f; x_t^h, a_t^h)^2),\\  \rE[\exp(-c\hatd_t^h(g,f) \mid x_t^h, a_t^h] &\leq \exp(-c(1-2c\sigma^2)\BE(g,f; x_t^h, a_t^h)^2)
\end{align*}
\label{lem:BE-exp}
\end{lemma} 
\begin{proof}
We note that
\[
\hatd_t^h(g,f) = \BE^h(g,f;x_t^h,a_t^h)^2
+ 2 \hat{\epsilon}_t^h(f)\BE^h(g,f;x_t^h,a_t^h) .
\]
This implies the first result. 
We note that Lemma~\ref{lem:noise} implies that
\[
\BE^h(g,f;x,a) = \rE_{r_t^h,x_{t}^{h+1}|x_t^h=x,a_t^h=a}
\hatd_t^h(g,f) .
\]
This implies the second result. 
Moreover, Lemma~\ref{lem:noise} implies that
\begin{align*}
    \rE_{r_t^h,x_{t}^{h+1}|x_t^h=x,a_t^h=a}
\hatd_t^h(g,f)^2
=& \BE^h(g,f;x,a)^4 + 4 \rE_{r_t^h,x_{t}^{h+1}|x_t^h=x,a_t^h=a} \hat\epsilon_t^h(g,f)^2 \BE^h(g,f;x,a)^2 \\
\leq& (1+4\sigma^2) \BE^h(g,f;x,a)^2 .
\end{align*}
This implies the third result. The last two inequalities follow by using the first part of the lemma and then using the sub-Gaussian exponential bound from Lemma~\ref{lem:noise} with $b_t = 2c\BE(g,f; x_t^h, a_t^h)$ and $b_t = -2c\BE(g,f; x_t^h, a_t^h)$ respectively for the two bounds.
\end{proof}

\subsection{Convergence of outer updates}
\label{sec:online-outer}

In this section, we build the necessary results to prove Lemma~\ref{lem:regret-hatc-Z}. 
We begin with the following bound for $\hatc_t^h(f)$.
\begin{lemma}
We have 
\[
\hatd_t^h(f) - \gamma \psi(\gamma\sigma^2)\rE_{g|f,S_{t-1}} \hatd_t^h(g,f)^2\leq \hatc_t^h(f) \leq  \hatd_t^h(f) .
\]
Assume that
\[
\alpha'=\exp(\gamma(1-2\gamma\sigma^2))/(1-2\gamma\sigma^2) >0 .
\]
Then for any $f$ that may depend on $S_{t-1}$:
\[
 \rE_{x_t,a_t,r_t|h_t=h} \; \rE_{g|f,S_{t-1}} \BE^{h}(g,f;x_t^{h},a_t^{h})^2
\leq  \alpha'  \rE_{x_t,a_t,r_t|h_t=h} \;  \hatc_t^{h}(f) .
\]
\label{lem:hatdc0}
\end{lemma}
\begin{proof}
We have
\begin{align*}
    -\gamma \hatc_t^h(f)
=&    \ln \;\rE_{g \sim p(g|f,S_{t-1})} \exp(-\gamma \hatd_t^h(g,f))\\
\leq&\rE_{g \sim p(g|f,S_{t-1})}  \exp(-\gamma \hatd_t^h(g,f))-1\tag{$\ln z \leq z-1$}\\
=&\rE_{g \sim p(g|f,S_{t-1})}  \left[-\gamma \hatd_t^h(g,f) + \gamma^2\psi(-\gamma\hatd_t^h(g,f)) \hatd_t^h(g,f)^2\right] \tag{$\psi(z)$ is increasing in $z$}\\
\leq&\rE_{g \sim p(g|f,S_{t-1})}
\left[-\gamma \hatd_t^h(g,f) + \gamma^2\psi(\gamma\sigma^2) \hatd_t^h(g,f)^2\right]. \tag{Jensen's inequality}
\end{align*}

Now conditioned on $h_t=h$, we have
\begin{align*}
&- \gamma \rE_{x_t,a_t,r_t}  \;  \hatc_t^{h}(f)\\
\leq & \rE_{x_t,a_t,r_t} \ln \;\rE_{g \sim p(g|f,S_{t-1})} \rE_{x_{t}^{h+1},r_t^h|x_t^h,a_t^h}e^{-\gamma \hatd_t^h(g,f)} \tag{Jensen's inequality}\\
\leq& \rE_{x_t,a_t,r_t} \ln \;\rE_{g \sim p(g|f,S_{t-1})}  e^{-\gamma (1-2\gamma \sigma^2)\BE^h(g,f,x_t,a_t)^2} \tag{Lemma~\ref{lem:BE-exp}}\\
\leq& \rE_{x_t,a_t,r_t}\; \left[\rE_{g \sim p(g|f,S_{t-1})}  e^{-\gamma (1-2\gamma \sigma^2)\BE^h(g,f,x_t,a_t)^2} -1\right]\tag{$\ln z \leq z-1$}\\
\leq& -\rE_{x_t,a_t,r_t}\; \rE_{g \sim p(g|f,S_{t-1})}  e^{-\gamma (1-2\gamma \sigma^2)} \gamma(1-2\gamma\sigma^2) \BE^h(g,f,x_t,a_t)^2 .
\end{align*}
Here the last inequality used
$e^{-z}-1 \leq -e^{-z'} z$ for $0 \leq z \leq z'$. 
This implies the desired bound.
\end{proof}

\begin{lemma}
We have
\[
\hatc_t \leq  \rE_{f|S_{t-1}}   \left[-\frac{\lambda}{\eta} \Delta f (x_t^1) + \hatd_t^{h_t}(f,f) - \hatc_{t}^{h_t}(f)\right] ,
\]
and 
\[
|\hatc_t| \leq (\lambda/\eta) + (1+\sigma)^2 .
\]
\label{lem:hatc-bound}
\end{lemma}
\begin{proof}
  The first inequality follows from Jensen's inequality. The second
  inequality follows from Lemma~\ref{lem:BE-exp}.
\end{proof}

We also require a bound on the log-partition function. 
\begin{lemma}
  If $2\gamma \sigma^2<1$, then for all $t \leq T$:
 \[
 \rE \; Z(S_t) \leq  \lambda \epsilon T +4
 \eta T \epsilon^2 +  \kappa(\epsilon) .
 \]
 \label{lem:Z}
 \end{lemma}
 \begin{proof}
 For any probability distribution $p$ on $\cF$, we have
 \begin{align*}
     \rE_{S_t} \; Z(S_t)\leq &  \rE_{f \sim p} \rE_{S_t}
     \left(\sum_{s=1}^t -\lambda \Delta f(x_s^1)+\eta \sum_{s=1}^t [\hatd_s^{h_s}(f,f) - \hatc_{s}^{h_s}(f)]\right)
     + \rE_{f \sim p} \ln \frac{p(f)}{p_0(f)} \\
     \leq& \rE_{f \sim p} \rE_{S_t}
     \left(\sum_{s=1}^t -\lambda \Delta f(x_s^1)+\eta \sum_{s=1}^t \sum_{h=1}^H \hatd_t^h(f,f) \right)
     + \rE_{f \sim p} \ln \frac{p(f)}{p_0(f)} .
 \end{align*}
 The first inequality used the fact that $p(f|S_t)$ is the minimizer of the right hand side over $p$. 
 The second inequality used Lemma~\ref{lem:hatdc0} and $2\gamma \sigma^2<1$.
 Using Lemma~\ref{lem:BE-exp}, we further obtain
 \[
 \rE_{S_t} Z(S_t)\leq
 \rE_{f \sim p} 
     \left[\rE_{S_t}\left(\sum_{s=1}^t -\lambda \Delta f(x_s^1)+\eta \sum_{s=1}^t  \BE^{h_s}(f,f,x_s,a_s)^2 \right)
     + \rE_{f \sim p} \ln \frac{p(f)}{p_0(f)} \right] .
 \]
 We now recall our definition of the set $\cF(f,\epsilon)$ for any $f\in\cF$ from Definition~\ref{def:kappa} as the set of all functions which capture the Bellman error of $f$ up to an error $\epsilon$. Then we have $\forall f \in \cF(Q_\star,\epsilon)=\prod_h \cF(Q_\star^h,\epsilon)$
 \[
 |\Delta f(x_s^1)| \leq \epsilon, 
 \qquad \BE^h(f,f,x_t,a_t) \leq 2 \epsilon .
 \]
  We can now take 
 \[
 p(f) = \frac{p_0(f) I(f \in \cF(Q_\star,\epsilon))}{p_0(\cF(Q_\star,\epsilon))} .
 \]
 It implies the desired bound.
 \end{proof}

We are now ready to prove Proposition~\ref{lem:regret-hatc-Z}.

\mypar{Proof of Proposition~\ref{lem:regret-hatc-Z}}~\\
\begin{proof}
Let us define $\alpha=6\eta \sigma^2 < 1$
Lemma~\ref{lem:hatdc0} implies that
\[
\hatc_t^{h_t}(f) \leq \hatd_t^{h_t}(f) .
\]
Therefore 
\begin{align*}
    &\rE [Z(S_{t-1}) - Z(S_t)] = \rE \left[\ln \rE_{f|S_{t-1}} \exp\left(\lambda \Delta f(x_t^1) -\eta [\hatd_t^{h_t}(f,f) - \hatc_t^{h_t}(f)] \right) \right]\\
    \leq& \rE \left[\ln \rE_{f|S_{t-1}} \exp\left(\lambda \Delta f(x_t^1)-\eta [\hatd_t^{h_t}(f,f) - \hatd_t^{h_t}(f)] \right) \right] \tag{Lemma~\ref{lem:hatdc0}}\\
    \leq& \frac{1}{3}\left[\ln \rE \rE_{f|S_{t-1}} \exp\left(3\lambda \Delta f(x_t^1) \right) +\ln \rE \rE_{f|S_{t-1}}  \exp\left(-3\eta \hatd_t^{h_t}(f,f)\right) +\ln \rE \rE_{f|S_{t-1}}\exp\left( 3\hatd_t^{h_t}(f)\right) \right],
\end{align*}
where the last inequality follows from Jensen's inequality to show that $\rE[XYZ] \leq \sqrt[3]{\rE[X^3]\rE[Y^3]\rE[Z^3]}$ for non-negative random variables $X,Y,Z$. Now we further simplify each term by using the last two bounds in Lemma~\ref{lem:BE-exp}, by taking a conditional expectation with respect to $x_t^{h+1}, r_t^h$, conditioned on $x_t^h, a_t^h$ to obtain 
\begin{align*}
    &\rE [Z(S_{t-1}) - Z(S_t)]\\
    \leq& \frac{1}{3}\rE \rE_{f|S_{t-1}} \left(3\lambda \Delta f(x_t^1)+ 9\lambda^2/2\right) 
    +\frac13 \ln \rE \rE_{f|S_{t-1}} \exp \left(-(3\eta- 18\eta^2\sigma^2) \BE^{h_t}(f,f,x_t,a_t)^2\right)\\
    &+\frac13 \ln \rE \rE_{f|S_{t-1}} \exp\left( (3\eta+18\eta^2\sigma^2)\rE_{g|f,S_{t-1}} \BE^{h_t}(g,f,x_t,a_t)^2\right) \\
    \leq& \frac{1}{3}\rE \rE_{f|S_{t-1}} \left(3\lambda \Delta f(x_t^1)+ 9\lambda^2 \right) 
    -\eta (1-\alpha) e^{-3\eta(1-\alpha)} \rE \rE_{f|S_{t-1}}  \BE^h(f,f,x_t,a_t)^2\\
    &+\eta (1+\alpha) e^{3\eta(1+\alpha)} \rE \rE_{f|S_{t-1}} \rE_{g|f,S_{t-1}} \BE^h(g,f,x_t,a_t)^2.
\end{align*}
The last inequality used  $\alpha = 6\eta \sigma^2 < 1$ and
\[
    \ln \rE_\xi \exp(-z(\xi)) \leq \rE_\xi \exp(-z(\xi)) - 1 \leq - \exp(-\max_\xi z(\xi)) \rE_\xi z(\xi),
\]
again using $e^{-z} - 1 \leq -ze^{-z'}$ for $0 \leq z \leq z'$. Similarly, we also get $\ln \rE_\xi \exp(z(\xi)) \leq \exp(\max_\xi z(\xi)) \rE_\xi z(\xi)$ for $z(\xi)\geq 0$. 
By summing over $t=1$ to $t=T$, and note that $Z(S_0)=0$, 
 we obtain the desired bound. 
\end{proof}

\subsection{Simplified convergence analysis of the inner updates}
\label{sec:online-inner-slow}

In this section, we establish the following result, which is simpler to prove than Proposition~\ref{lem:inner}.

\begin{proposition}[Inner loop convergence]
Suppose $\gamma < 1/2$ and $\lambda+\eta(1+\sigma^2)\leq 0.1$. Then we have for all $\epsilon > 0$:
\begin{align*}
    \frac{1}{2}\rE \rE_{f,g|S_{t-1}} \BE^{h_t}(g,f;x_t^{h_t},a_t^{h_t})^2 \leq &
    \epsilon T (\epsilon+3) (6+10 \lambda T + 90\eta T) +  25(\lambda +9 \eta) T \\
    & +  \frac{1 + 3\lambda T + 25 \eta T}{\gamma} (\kappa(\epsilon)+\kappa'(\epsilon)) .
\end{align*}
\label{lem:inner-slow}
\end{proposition}

We first show how this immediately yields an $\order(T^{-1/8})$ sample complexity of our algorithm, before proving the statement through a series of lemmas. Though we eventually supercede this analysis with a sharper one, several intermediate results will be reused and we believe that the simpler analysis of this proposition illustrates the main ideas of solving the nested minimax setup.

We now state a form of Theorem~\ref{thm:main}, using Proposition~\ref{lem:inner-slow}.

\begin{theorem}
    Under Assumptions~\ref{ass:realizable}-\ref{assumption:embedding}, suppose we run \alg (Algorithm~\ref{alg:online_TS}) with some parameters $\gamma \leq \nicefrac{1}{36}$ and $\eta \leq 0.01$. Then choosing any $\epsilon \leq 0.6/T^2$, we have  
    \begin{align*}
    \rE\; \sum_{t=1}^T \regret(f_t,x_t^1) &= \order\left(\frac{\eta}{\lambda}(1+\lambda T + \eta T)(\kappa(\epsilon) + \kappa'(\epsilon)) + \lambda T + \frac{\kappa(\epsilon)}{\lambda} + \tilde{\epsilon}(\lambda/\eta) T\right),
    \end{align*}
    where $\tilde{\epsilon}(\lambda/\eta) = \inf_{\mu>0} \left[ 8(\lambda/\eta) \dc(\epsilon_2)H\mu^{2} + 2\mu H \epsilon_2 B_2 + \mu^{-1} \br(\epsilon_1) +\epsilon_1 H B_1    \right]$.
\label{thm:main-slow}
\end{theorem}

To get this result, we set $\eta < 0.01$ so that $(1-6\eta)\exp(-3(1-6\eta)) \geq 0.5$ along with the stated values of $\epsilon$ and $\gamma$. Plugging these into the bounds of Propositions~\ref{prop:value-decomposition}, \ref{lem:decouple}, \ref{lem:regret-hatc-Z} and \ref{lem:inner-slow}, and simplifying gives the result of Theorem~\ref{thm:main-slow}. Further assuming the conditions of Corollary~\ref{cor:finite}, we can choose $\mu = \left(d_1\eta/(d_2\lambda H^2)\right)^{1/3}$, $\eta = 1/\sqrt{T}$, $\gamma = 1/36$ and $\lambda = T^{-7/8}(d_1^2d_2H^2)^{-1/4} (\ln N)^{3/4}$ to get a sample complexity of $\order\left(H(\ln N)^{1/4}(d_1^2 d_2 H)^{1/4}\,T^{-1/8}\right)$.

We begin the proof of Proposition~\ref{lem:inner-slow} with a result that carries out a potential function analysis for the inner updates. We recall our definition of $q_t = p(g|f,S_t)$ in line~\ref{line:inner-update} of Algorithm~\ref{alg:online_TS}, which will be repeatedly used in this section. 

\begin{lemma}
Assume that $2\gamma \sigma^2 \leq 1$.
Let
\[
\pp(g|f) = \frac{p_0(g) I(g \in \cF(\epsilon,f))}{p_0(\cF(\epsilon,f))} , \quad \mbox{and}\quad \hat{H}_t(f) =  \rE_{g \sim \pp(\cdot|f)}\ln \frac{\pp(g|f)}{p(g|f,S_t)}.
\]
Then for all $t$:
\[
\rE_{S_t} \; \sup_{f} \hat{H}_t(f) \leq  \ln \rE_{S_t} \; \sup_{f} \exp(\hat{H}_t(f)) \leq
\kappa(\epsilon) +
\kappa'(\epsilon) + 4\gamma \epsilon (\epsilon+1+\sigma) t .
\]
\label{lem:hatH}
\end{lemma}
\begin{proof}
Let
\[
A_t(g,f) = \gamma \sum_{s=1}^{t} \hatd_s^{h_s}(g,f) .
\]
Then for each $g \in \cF(\epsilon,f)$, we have:
\[
-2\gamma t \epsilon\sigma \leq A_t(g,f) \leq \gamma (\epsilon^2+2\epsilon \sigma) t .
\]
Let $f_1,\ldots,f_N$ be a cover of $\cF$ so that for any $f \in \cF$,
$\exists j$  such that $|f^h(x)-f_j^h(x)| \leq \epsilon$  for all $x$. We
know that $\ln N \leq \kappa'(\epsilon)$ (Definition~\ref{def:kappa}). This implies that for all $S_t$, $\exists j\in[N]$:
\[
 -A_t(g,f) \leq -A_t(g,f_j) +
  \gamma  (\epsilon^2+2\epsilon(1+\sigma)) t .
\]
It follows that with $p_0(g)=\prod_{h=1}^H p_0^h(g^h)$, we obtain
\begin{align*}
&\ln \rE_{S_{t}} \sup_{f \in \cF}\rE_{g \sim p_0}
\exp\left(-A_t^h(g,f)\right) \\
\leq &\gamma\epsilon t(\epsilon+2+2\sigma) +  \ln \rE_{S_{t}} \sup_j   \rE_{g \sim p_0}
\exp\left(-\gamma \sum_{s=1}^{t} \hatd_s^{h_s}(g,f_j) \right) \\
  \leq &\gamma \epsilon t(\epsilon+2+2\sigma) +    \ln \sum_j 
        \rE_{g \sim p_0}\rE_{S_{t}}
 \exp\left(-\gamma\sum_{s=1}^t \hatd_s^{h_s}(g,f_j)\right) \\
\stackrel{(a)}{\leq} &\gamma\epsilon t(\epsilon+2+2\sigma) +   \ln \sum_j \rE_{g \sim p_0}
 \rE_{S_t}\exp\left(-\gamma(1-2\gamma \sigma^2)\BE^h(g,f_j,x_t^{h_t},a_t^{h_t})^2 -\gamma \sum_{s=1}^{t-1}
       \hatd_s^{h_s}(g,f_j) \right) \\
       \leq &\gamma \epsilon t(\epsilon+2+2\sigma) +    \ln \sum_j 
        \rE_{g \sim p_0}\rE_{S_{t}}
 \exp\left(-\gamma\sum_{s=1}^{t-1} \hatd_s^{h_s}(g,f_j)\right) \\
       \leq &\cdots \\
  \leq & \gamma\epsilon t(\epsilon+2+2\sigma) +  \ln N .
\end{align*}
The first inequality used covering property.  Inequality (a) uses the last bound in Lemma~\ref{lem:BE-exp}. 
We thus have
\begin{align*}
     &\rE_{S_t}\sup_{f\in\cF}\hat{H}_t(f) \leq  \ln \rE_{S_t}\sup_{f\in\cF}\exp(\hat{H}_t(f))\\
     =&\ln \rE_{S_t}\sup_{f\in\cF}\exp\left(\rE_{g \sim \pp(\cdot|f)}\ln \frac{\pp(g|f)}{p(g|f,S_t)}\right) \\
     = & \ln \rE_{S_t}\sup_{f\in\cF}\exp\left[\rE_{g \sim \pp(\cdot|f)} \ln \frac{p_0(g)}{p_0(g)\exp(-A_t(g,f))}
     + \ln \frac{1}{p_0(\cF(\epsilon,f))} +  \ln \rE_{g \sim p_0}
\exp\left(-A_t(g,f)\right)\right]\\
\leq &  \gamma (\epsilon^2+2\epsilon \sigma)  t
     +  \kappa(\epsilon) + 2 \gamma \epsilon t(\epsilon+2+2\sigma) + \kappa'(\epsilon) .
\end{align*}
The first inequality used Jensen's inequality.
This implies the result. 
\end{proof}

We give an upper bound on the log-partition function $\hatc_t^h(f)$. Note that this is the part of our analysis which relies on $\eta$ being smaller than $\gamma$, and leads to a loss in rates. It is possible to sharpen this analysis through a more careful self-bounding argument. The calculation will be more complex, and we leave the 
refined analysis to the next section.

\begin{lemma}
Let
\[
\alpha'' = 2(\lambda/\eta+(1+\sigma)^2) \exp(2\lambda+2\eta(1+\sigma)^2) .
\]
Then
\begin{align*}
\rE \; \sum_{t=1}^T \rE_{f|S_{t-1}} \hatc_t^{h_t}(f) \leq & \epsilon T (\epsilon+1+\sigma) (6+4\eta \alpha''T) +  \eta T (1+\sigma)^2 \alpha''
+  (\gamma^{-1} +\eta\alpha''T/\gamma) (\kappa(\epsilon)+\kappa'(\epsilon)) .
\end{align*}
\label{lem:hatc}
\end{lemma}
\begin{proof}
We know that 
\[
p(f|S_t)= p(f|S_{t-1}) \exp(\lambda \Delta f(x_t^1)-\eta [(\hatd_t^{h_t}(f,f)-\hatc_t^{h_t}(f))-\hatc_t]).
\]
Since
\[
|\lambda \Delta f(x_t^1)-\eta[(\hatd_t^{h}(f,f)-\hatc_t^{h}(f))-\hatc_t]| \leq 2 \lambda + 2 \eta (1+\sigma)^2 ,
\]
by Lemma~\ref{lem:BE-exp}, we obtain by using $|\exp(z)-1| \leq |z|\exp(|z|)$ with $z= \lambda \Delta f(x_t^1)-\eta[\hatd_t^{h}(f,f)-\hatc_t^{h}(f))-\hatc_t]$
\begin{equation}
  \bigg|\exp(\lambda \Delta f(x_t^1)-\eta
    [(\hatd_t^{h_t}(f,f)-\hatc_t^{h_t}(f))-\hatc_t])-1\bigg|
  \leq \eta \alpha'' .
  \label{eq:hatd-hatc}
\end{equation}
Let $p_t=p(f|S_t)$.
We have 
\begin{align*}
&\rE_{f\sim p_t} \hat{H}_t(f) - \rE_{f\sim p_{t-1}}\hat{H}_{t-1}(f)\\ =& \rE_{f\sim p_t-p_{t-1}} \rE_{g \sim \pp(\cdot|f)} \ln \frac{\pp(\cdot|f)}{p(g|f,S_{t-1})}
-\rE_{f\sim p_t-p_{t-1}} 
\rE_{g \sim \pp(\cdot|f)} \ln \frac{p(g|f,S_t)}{p(g|f,S_{t-1})}
\\
&- \rE_{f\sim p_{t-1}} \rE_{g \sim \pp(\cdot|f)} \ln \frac{p(g|f,S_{t})}{p(g|f,S_{t-1})}.
\end{align*}
Now we observe that 
\[
    p(g|f,S_t) = p(g|f,S_{t-1})\exp(-\gamma(\hatd_t^{h_t}(g,f) - \hatc_t^{h_t}(f))),
\]
which allows us to further rewrite
\begin{align*}
&\rE_{f\sim p_t} \hat{H}_t(f) - \rE_{f\sim p_{t-1}}\hat{H}_{t-1}(f)\\
=&\rE_{f\sim p_{t-1}} 
\left[e^{\lambda \Delta f(x_t^1)-\eta [(\hatd_t^{h_t}(f,f)-\hatc_t^{h_t}(f))-\hatc_t^{h_t}]}-1\right]
\rE_{g\sim \pp(\cdot|f)} \ln \frac{\pp(g|f)}{p(g|f,S_{t-1})}\\
&-\gamma \rE_{f\sim p_{t-1}}
\left[e^{\lambda \Delta f(x_t^1)-\eta [(\hatd_t^{h_t}(f,f)-\hatc_t^{h_t}(f))-\hatc_t^{h_t}]}-1\right]
\rE_{g\sim\pp(\cdot|f)}[-\hatd_t^{h_t}(g,f)+\hatc_t^{h_t}(f)] \\
&- \gamma \rE_{f\sim p_{t-1}} \rE_{g\sim\pp(\cdot|f)} [-\hatd_t^{h_t}(g,f)+\hatc_t^{h_t}(f)]\\
\leq & \eta \alpha'' \rE_{f\sim p_{t-1}}\hat{H}_{t-1}(f)+  \eta\gamma (1+\sigma)^2 \alpha'' - \gamma   \rE_{f|S_{t-1}} \rE_{g\sim\pp(\cdot|f)} [-\hatd_t^{h_t}(g,f)+\hatc_t^{h_t}(f)] .
\end{align*}
The last inequality used $|\hatd_t^{h_t}(g,f)+\hatc_t^{h_t}(f)|\leq (1+\sigma)^2$, along with our earlier inequality~\eqref{eq:hatd-hatc} and the observation that $\rE_{g\sim \pp(\cdot|f)} \ln \frac{\pp(g|f)}{p(g|f,S_{t-1})}$ is a KL divergence, and hence non-negative.
By rearranging the terms, we obtain
\begin{align*}
\rE_{f\sim p_{t-1}} \hatc_t^{h_t}(f) \leq &
\rE_{f\sim p_{t-1}} \rE_{g\sim\pp(\cdot|f)} \hatd_t^{h_t}(g,f) 
+  \eta (1+\sigma)^2 \alpha''
 + \frac{1}{\gamma} [(1+\eta\alpha'')\rE_{f\sim p_{t-1}}\hat{H}_{t-1}(f) - \rE_{f\sim p_t}\hat{H}_{t}(f)] .
\end{align*}
Note that for $g \in \cF(\epsilon,f)$, we have $\hatd_t^{h_t}(g,f) \leq
\epsilon^2+2\epsilon\sigma$. By summing over $t$, we obtain
\begin{align*}
\rE \; \sum_{t=1}^T \rE_{f|S_{t-1}} \hatc_t^{h_t}(f) \leq & \epsilon(\epsilon+2\sigma)T +  \eta T (1+\sigma)^2 \alpha'' + \frac{1}{\gamma} (1+ \eta \alpha''T) \rE\sup_{t \leq T,f\in\cF} \hat{H}_t(f) .
\end{align*}
We can now obtain the desired bound using Lemma~\ref{lem:hatH}.
\end{proof}

We are now ready to prove Proposition~\ref{lem:inner-slow}.

\mypar{Proof of Proposition~\ref{lem:inner-slow}}~\\

\begin{proof}

The proof essentially follows from Lemma~\ref{lem:hatc}. We have from Lemma~\ref{lem:hatdc0}
  \begin{align*}
    &\sum_{t=1}^T \rE \rE_{f,g|S_{t-1}} \BE^{h_t}(g,f;x_t^{h_t},a_t^{h_t})^2\\ \leq&  \alpha'  \sum_{t=1}^T \rE \rE_{f,g|S_{t-1}} \hatc_t^{h_t}\\
\leq& \alpha' 
  \left[   \epsilon T (\epsilon+1+\sigma) (6+4\eta \alpha''T) +  \eta T (1+\sigma)^2 \alpha'' +  (\gamma^{-1} +\eta\alpha''T/\gamma) (\kappa(\epsilon)+\kappa'(\epsilon))\right]\\
    \leq&    2 
  \left[   \epsilon T (\epsilon+3) (6+10 \lambda T + 90\eta T) +  25(\lambda +9 \eta) T +  (1/\gamma)(1 + 3\lambda T + 25 \eta T) (\kappa(\epsilon)+\kappa'(\epsilon))\right]
  \end{align*}
  The last inequality used our assumptions on the various parameters, which imply that $\alpha <0.25$, $\alpha'<2$, and $\alpha'' < 2.5((\lambda/\eta)+9)$.
\end{proof}

\subsection{Proof of Proposition~\ref{lem:inner}}
\label{sec:online-inner}

In the following, we derive a refinement of Lemma~\ref{lem:hatc}, which allows us to prove Proposition~\ref{lem:inner}.
 In order to avoid complex constant calculations, we will use the $\order(\cdot)$ notation that hides absolute constants. Here we take $\sigma=\order(1), \eta=\order(1), \gamma=\order(1)$, and $\epsilon=\order(1)$. Consequently, we also have that $\alpha = \order(1)$ and $\alpha' = \order(\gamma) = \order(1)$. We also repeatedly use that for $b = \order(1), \exp(b) \leq 1 + \theta b$ for $\theta \leq b$ by the intermediate value theorem, so that $\exp(b) - 1 = \order(b)$ when $b = \order(1)$. In particular, for the function $\psi(z)$ defined at the start of this section, we have
\begin{equation}
    z^2\psi(z) = e^z - z - 1 = \order(z^2), \quad \mbox{when $z = \order(1)$}.
    \label{eq:psi-bound}
\end{equation}

We have the following high probability bound for the entropy considered in Lemma~\ref{lem:hatH-prob}.
\begin{lemma}
Under the Assumption of Lemma~\ref{lem:hatH}, for each $t$, event $A_t$, defined below, holds 
with probability at least $1-1/T^2$ over $S_t$:
\[
A_t = \left\{ \sup_{f\in\cF} \hat{H}_t(f) \leq \kappa(\epsilon) +
\kappa'(\epsilon) + 6\gamma \epsilon (\epsilon+1+\sigma) t +  2\ln T \right\}.
\]
\label{lem:hatH-prob}
\end{lemma}
\begin{proof}
From Lemma~\ref{lem:hatH}, we obtain for each $t \leq T$:
\[
\ln \rE_{S_t} \; \sup_{f\in\cF} \exp(\hat{H}_t(f)) \leq
\kappa(\epsilon) +
\kappa'(\epsilon) + 6\gamma \epsilon (\epsilon+1+\sigma) t .
\]
Using Markov's inequality, we obtain for each $t \leq T$, with probability $1-1/T^2$, 
\[
\sup_{f} \exp(\hat{H}_t(f)) \leq
\kappa(\epsilon) +
\kappa'(\epsilon) + 6\gamma \epsilon (\epsilon+1+\sigma) t + 2 \ln T .
\]
This leads to the bound.
\end{proof}

We also have the following uniform bound for the entropy considered in Lemma~\ref{lem:hatH-prob}.
\begin{lemma}
Under the Assumption of Lemma~\ref{lem:hatH}, for all $t \leq T$ and $f$:
\[
\hat{H}_t(f) = \order(\kappa(\epsilon) + T) .
\]
\label{lem:hatH-uniform}
\end{lemma}
\begin{proof}
We note that in the proof of Lemma~\ref{lem:hatH}, we can simply bound
\[
|A_t(g,f)| = \order(T) .
\]
We thus have
\begin{align*}
     \hat{H}_t(f) 
     =&\rE_{g \sim \pp(\cdot|f)}\ln \frac{\pp(g|f)}{p(g|f,S_t)}\\
     = & \rE_{g \sim \pp(\cdot|f)} \ln \frac{p_0(g)}{p_0(g)\exp(-A_t(g,f))}
     + \ln \frac{1}{p_0(\cF(\epsilon,f))} +  \ln \rE_{g \sim p_0}
\exp\left(-A_t(g,f)\right)\\
= &    \kappa(\epsilon) + \order(T) .
\end{align*}
This implies the result. 
\end{proof}

\begin{lemma}
We have
\[
\rE_{x_t,a_t,r_t|h_t=h} |\hatc_t^h(f)|^2 = \order(\rE_{x_t,a_t,r_t|h_t=h} \hatc_t^h(f)) .
\]
\label{lem:hatc2-1}
\end{lemma}
\begin{proof}
From Lemma~\ref{lem:hatdc0}, along with~\eqref{eq:psi-bound}, we obtain
\begin{align*}
|\hatc_t^h(f)| =& \order( |\rE_{g|f,S_{t-1}} \hatd_t^h(g,f)| + \gamma \rE_{g|f,S_{t-1}} \hatd_t^h(g,f)^2)\\
=& \order( \rE_{g|f,S_{t-1}} |\hatd_t^h(g,f)|) . \tag{since $\gamma = \order(1)$ by assumption}
\end{align*}
It follows that 
\begin{align*}
\rE_{x_t,a_t,r_t|h_t=h} |\hatc_t^h(f)|^2 =& \rE_{x_t,a_t,r_t|h_t=h} 
\order( \rE_{g|f,S_{t-1}} |\hatd_t^h(g,f)|^2) \\
=& \rE_{x_t,a_t,r_t|h_t=h} 
\order( \rE_{g|f,S_{t-1}} \BE^h(g,f;x_t^h,a_t^h)^2) \tag{Lemma~\ref{lem:BE-exp}} \\
=& \rE_{x_t,a_t,r_t|h_t=h} 
\order( \rE_{g|f,S_{t-1}} \hatc_t^h(f)) ). \tag{Lemma~\ref{lem:hatdc0}}
\end{align*}
This proves the desired result. 
\end{proof}

\begin{lemma}
We have
\begin{align*}
\rE_{x_t,a_t,r_t|h_t=h} |\hatc_t^{h_t}|^2 =&  \rE_{x_t,a_t,r_t|h_t=h} \rE_{f|S_{t-1}} \order((\lambda/\eta)^2 |\Delta f(x_t^1)|+ |\BE^{h}(f,f;x_t^h,a_t^h)|^2 + \hatc_{t}^{h}(f)) .
\end{align*}
\label{lem:hatc2-2}
\end{lemma}
\begin{proof}
We have from the definition of $\hatc_t^h$:
\begin{align*}
|\hatc_t^{h_t}| \leq& \rE_{f|S_{t-1}} \order((\lambda/\eta) |\Delta f(x_t^1)|+|\hatd_t^{h_t}(f,f)| + |\hatc_{t}^{h_t}(f)|)  .
\end{align*}
Therefore we obtain 
\begin{align*}
  \rE_{x_t,a_t,r_t|h_t=h} |\hatc_t^{h_t}|^2 =  &
  \rE_{x_t,a_t,r_t|h_t=h} \rE_{f|S_{t-1}} \order((\lambda/\eta)^2 |\Delta f(x_t^1)|^2 + |\hatd_t^{h_t}(f,f)|^2 + |\hatc_{t}^{h_t}(f)|^2)  \\
=& \rE_{x_t,a_t,r_t|h_t=h} \rE_{f|S_{t-1}} \order((\lambda/\eta)^2|\Delta f(x_t^1)|+|\BE^{h}(f,f;x_t^h,a_t^h)|^2 + \hatc_{t}^{h}(f)) ,
\end{align*}
where the second inequality uses Lemma~\ref{lem:BE-exp} and Lemma~\ref{lem:hatc2-1}. 
\end{proof}

\begin{lemma}
There exists an absolute constant $c$ such that when $\eta \leq c$,  
then for all $f\in\cF$:
\begin{align*}
&\rE_{x_t,a_t,r_t|h_t=h} \left[e^{\lambda \Delta f(x_t^1)-\eta [(\hatd_t^{h_t}(f,f)-\hatc_t^{h_t}(f))-\hatc_t^{h_t}]}-1\right]\\
\leq& \eta \order\left(
\rE_{x_t} \rE_{f' |S_{t-1}}\frac{\lambda}{\eta} (|\Delta f(x_t^1)|+|\Delta f'(x_t^1)|)
+\rE_{x_t,a_t,r_t|h_t=h}
  \hatc_t^{h}(f)
+ \rE_{x_t,a_t,r_t|h_t=h}\rE_{f' |S_{t-1}} \BE^{h}(f',f',x_t^{h},a_t^h)^2\right).
\end{align*}
\label{lem:hatc-refined-0}
\end{lemma}
\begin{proof}
We have
\begin{align*}
   & \rE_{x_t,a_t,r_t|h_t=h} \left[e^{\lambda \Delta f(x_t^1)-\eta [(\hatd_t^{h_t}(f,f)-\hatc_t^{h_t}(f))-\hatc_t^{h_t}]}-1\right]\\
\leq&  \rE_{x_t,a_t,r_t|h_t=h} \left[e^{\lambda \Delta f(x_t^1)-\eta \hatd_t^{h_t}(f,f)+\eta \hatd_t^{h_t}(f))+ \eta \rE_{f'|S_{t-1}} [-(\lambda/\eta) \Delta f'(x_t^1)+\hatd_t^{h_t}(f',f')-\hatc_t^{h_t}(f')]}-1\right] \tag{Lemmas~\ref{lem:hatdc0} and~\ref{lem:hatc-bound}} \\
\leq&  0.2 \rE_{x_t,a_t,r_t|h_t=h} \rE_{f'|S_{t-1}}\left[e^{5\lambda(\Delta f(x_t^1)-\Delta f'(x_t^1))}+e^{-5\eta \hatd_t^{h_t}(f,f)}+e^{5\eta \hatd_t^{h_t}(f)}+ e^{5\eta  \hatd_t^{h_t}(f',f')}+e^{-5\eta\hatc_t^{h_t}(f')}-5\right] \tag{Jensen's applied to $\exp(\cdot)$} .
\end{align*}
We now bound each of the four terms in turn. 
\[
\rE_{x_t,a_t,r_t|h_t=h} \rE_{f'|S_{t-1}} e^{5\lambda(\Delta f(x_t^1)-\Delta f'(x_t^1))}-1 \leq \order\left(\rE_{x_t} \rE_{f' |S_{t-1}}\lambda (|\Delta f(x_t^1)|+|\Delta f'(x_t^1)|)\right) .
\]
Note that conditioned on $h_t=h$, we have
\[
\rE_{x_t,a_t,r_t}e^{-5\eta \hatd_t^{h}(f,f)} - 1
\leq \rE_{x_t,a_t,r_t}\exp\left(-5\eta (1-10\eta\sigma^2)\BE^{h}(f,f,x_t^{h},a_t^h)^2\right) - 1 \leq 0,\tag{Lemma~\ref{lem:BE-exp}}
\]
and
\begin{align*}
\rE_{x_t,a_t,r_t}e^{5\eta \hatd_t^{h}(f)} - 1
\leq& \rE_{x_t,a_t,r_t}\exp\left(5\eta (1+10\eta\sigma^2)\rE_{g |f,S_{t-1}} \BE^{h}(g,f,x_t^{h},a_t^h)^2\right) - 1\tag{Lemma~\ref{lem:BE-exp}}\\
\leq& \eta \order\left( \rE_{x_t,a_t,r_t}\rE_{g |f,S_{t-1}} \BE^{h}(g,f,x_t^{h},a_t^h)^2\right) \\
\leq& \eta \order\left(  \hatc_t^{h}(f) \right). \tag{Lemma~\ref{lem:hatdc0}}
\end{align*}
Here the second inequality follows from the intermediate value theorem, since\\ $c = 5\eta (1+10\eta\sigma^2)\rE_{g |f,S_{t-1}} \BE^{h}(g,f,x_t^{h},a_t^h)^2 = \order(1)$, so that $e^c \leq 1 + \order(c)$. In the last inequality, we use $\alpha' = \order(1)$ in Lemma~\ref{lem:hatdc0}, since $\gamma = \order(1)$.

\begin{align*}
\rE_{x_t,a_t,r_t}e^{5\eta \hatd_t^{h}(f',f')} - 1
\leq& \rE_{x_t,a_t,r_t}\exp\left(4\eta (1+10\eta\sigma^2)\BE^{h}(f',f',x_t^{h},a_t^h)^2\right) - 1 \tag{Lemma~\ref{lem:BE-exp}}\\
\leq& \eta  \order\left(\rE_{x_t,a_t,r_t} \BE^{h}(f',f',x_t^{h},a_t^h)^2\right)
\end{align*}
and
\begin{align*}
\rE_{x_t,a_t,r_t}e^{-5\eta \hatc_t^{h}(f')} - 1
=& \rE_{x_t,a_t,r_t} \left[\psi(-5\eta \hatc_t^{h}(f')) (5\eta \hatc_t^{h}(f'))^2 - 5\eta \hatc_t^h(f')\right]\\
\leq& \rE_{x_t,a_t,r_t} \left(\order(\eta^2 \hatc_t^{h}(f')^2) - 5\eta \hatc_t^h(f')\right) 
\tag{$\psi(z)=\order(1)$} \\
\leq& \rE_{x_t,a_t,r_t} \left(\order(\eta^2 \hatc_t^{h}(f')) - 5\eta \hatc_t^h(f')\right) 
\tag{Lemma~\ref{lem:hatc2-1}} \\
\leq& 0 ,
\end{align*}
where the last inequality assumed that we choose $\eta$ small enough so that $\order(\eta^2) \leq 5 \eta$, 
and $\rE_{x_t,a_t,r_t} \hatc_t^{h}(f') \geq 0$, which follows from Lemma~\ref{lem:hatdc0}. 
\end{proof}

\begin{lemma}
For each $t \leq T$, we have
\begin{align*}
&\rE \; \rE_{f|S_{t-1}} 
\left[e^{\lambda \Delta f(x_t^1)-\eta [(\hatd_t^{h_t}(f,f)-\hatc_t^{h_t}(f))-\hatc_t^{h_t}]}-1\right]
 \hat{H}_{t-1}(f)\\
 =&
 \eta \rE \;  \rE_{f|S_{t-1}} \; \order\left(
  \left(\frac{\lambda}{\eta} |\Delta f(x_t^1)| +\hatc_t^{h_t}(f)
+  \BE^{h_t}(f,f,x_t^{h_t},a_t^{h_t})^2\right)(\kappa(\epsilon)+\kappa'(\epsilon)+\ln T) 
\right)
 + \order\left(\frac{(\eta+\lambda)}{T}(\kappa(\epsilon)+1)\right) .
 \end{align*}
 \label{lem:hatc-refined-1}
\end{lemma}
\begin{proof}
Let the indicator $I(A_t)$ denotes that the event of Lemma~\ref{lem:hatH-prob} holds.
By using the short hand notation
\[
w_t(f)=\rE_{x_t} \rE_{f' |S_{t-1}}\frac{\lambda}{\eta} (|\Delta f(x_t^1)|+|\Delta f'(x_t^1)|) ,
\]
we have
\begin{align*}
    &\rE \; \rE_{f|S_{t-1}} 
\left[e^{\lambda \Delta f(x_t^1)-\eta [(\hatd_t^{h_t}(f,f)-\hatc_t^{h_t}(f))-\hatc_t^{h_t}]}-1\right]
 \hat{H}_{t-1}(f) \\
 =& \rE \;\rE_{f|S_{t-1}} 
\left[\rE_{x_t, a_t, r_t} \big[e^{\lambda \Delta f(x_t^1)-\eta [(\hatd_t^{h_t}(f,f)-\hatc_t^{h_t}(f))-\hatc_t^{h_t}]}-1\big]\right]
 \hat{H}_{t-1}(f) \tag{Since $x_t, a_t, r_t$ are independent of $f$}\\
 =&
 \eta \rE \;\rE_{f|S_{t-1}}  \order\left(w_t(f)+\rE_{x_t, a_t, r_t}
  \hatc_t^{h_t}(f)
+ \rE_{x_t, a_t, r_t}\rE_{f' |S_{t-1}} \BE^{h_t}(f',f',x_t^{h_t},a_t^{h_t})^2\right)
\hat{H}_{t-1}(f) \tag{Lemma~\ref{lem:hatc-refined-0}}\\
=&
 \eta \rE \;\rE_{f|S_{t-1}}  \order\left(w_t(f)+
  \rE_{x_t, a_t, r_t}\hatc_t^{h_t}(f)+ \rE_{x_t, a_t, r_t}\rE_{f' |S_{t-1}} \BE^{h_t}(f',f',x_t^{h_t},a_t^{h_t})^2
 \right)
\hat{H}_{t-1}(f) I(A_t)\\
& + \eta\rE \;\rE_{f|S_{t-1}} \order\left(w_t(f)+\rE_{x_t, a_t, r_t}
  \hatc_t^{h_t}(f)
   + \rE_{x_t, a_t, r_t}\rE_{f' |S_{t-1}} \BE^{h_t}(f',f',x_t^{h_t},a_t^{h_t})^2
 \right)
\hat{H}_{t-1}(f) (1-I(A_t)) \\
=&
 \eta \rE \;  \order\left(\rE_{f|S_{t-1}} w_t(f)+
\rE_{f|S_{t-1}} 
  \hatc_t^{h_t}(f)
+  \rE_{f' |S_{t-1}} \BE^{h_t}(f',f',x_t^{h_t},a_t^{h_t})^2
\right)
\order(\kappa(\epsilon)+\kappa'(\epsilon)+\ln T)
\tag{Definition of $A_t$ in Lemma~\ref{lem:hatH-prob} and $w_t(f)\geq 0,\rE_{x_t,a_t,r_t} [\hatc_t^{h_t}(f)] \geq 0$}\\
& + (\eta+\lambda) \order(\kappa(\epsilon)+T) \rE (1-I({A}_t)) \tag{Lemma~\ref{lem:hatH-uniform}} \\
=&
 \eta \rE \;  \order\left(\rE_{f|S_{t-1}} w_t(f)+
\rE_{f|S_{t-1}} 
  \hatc_t^{h_t}(f)+\rE_{f' |S_{t-1}} \BE^{h_t}(f',f',x_t^{h_t},a_t^{h_t})^2
 \right)
 \order(\kappa(\epsilon)+\kappa'(\epsilon)+\ln T) \\
& + (\eta+\lambda) \order((\kappa(\epsilon)+1)/T) \tag{Probability of $A_t$ in Lemma~\ref{lem:hatH-prob}} .
\end{align*}
This implies the desired bound by noticing that conditioned on $S_{t-1}$, $f$ and $x_t$ are indepenent. 
\end{proof}
 
\begin{lemma}
 We have
\begin{align*}
&\rE_{x_t,a_t,r_t|h_t=h}\rE_{f|S_{t-1}}
\left|\left[e^{\lambda\Delta f(x_t^1)-\eta [(\hatd_t^{h_t}(f,f)-\hatc_t^{h_t}(f))-\hatc_t^{h_t}]}-1\right]
\rE_{g\sim\pp(\cdot|f)}[-\hatd_t^{h_t}(g,f)+\hatc_t^{h_t}(f)]\right|\\
=& \eta \rE_{x_t,a_t,r_t|h_t=h}\rE_{f|S_{t-1}} \order\left( 
(\lambda/\eta)^2 |\Delta f(x_t^1)|+|\BE^{h_t}(f,f;x_t^h,a_t^h)|^2+
\hatc_t^{h}(f)\right) + (\lambda+\eta) \order(\epsilon).
\end{align*}
\label{lem:hatc-refined-2}
\end{lemma}
\begin{proof}
We have
\[
|\lambda\Delta f(x_t^1)-\eta [(\hatd_t^{h_t}(f,f)-\hatc_t^{h_t}(f))-\hatc_t^{h_t}]|
  = \order(\lambda+\eta) ,
\]
and $\forall g \in \cF(\epsilon,f)$, we have $\hatd_t^{h_t}(g,f)=\order(\epsilon)$.
Therefore by the definition of $\pp(\cdot|f)$ in Lemma~\ref{lem:hatH}, we have
\begin{align*}
\bigg|[e^{\lambda\Delta f(x_t^1)-\eta [(\hatd_t^{h_t}(f,f)-\hatc_t^{h_t}(f))-\hatc_t^{h_t}]}-1]
\rE_{g\sim\pp(\cdot|f)}\hatd_t^{h_t}(g,f)\bigg|
=& (\lambda+\eta) \order (\epsilon) .
\end{align*}
Moreover,
\begin{align*}
& \rE_{x_t,a_t,r_t|h_t=h}\rE_{f|S_{t-1}} \bigg|[e^{\lambda\Delta f(x_t^1)-\eta [(\hatd_t^{h_t}(f,f)-\hatc_t^{h_t}(f))-\hatc_t^{h_t}]}-1]
\hatc_t^{h_t}(f)\bigg|\\
=& \eta \rE_{x_t,a_t,r_t|h_t=h}\rE_{f|S_{t-1}} \order\left([(\lambda/\eta)|\Delta f(x_t^1)|+|\hatd_t^{h_t}(f,f)|+|\hatc_t^{h_t}(f))|+|\hatc_t^{h_t}|] |\hatc_t^{h_t}(f)|\right)\\
=& \eta  \rE_{x_t,a_t,r_t|h_t=h}\rE_{f|S_{t-1}} \order\left((\lambda/\eta)^2|\Delta f(x_t^1)|^2+|\hatd_t^{h_t}(f,f)|^2+|\hatc_t^{h_t}(f))|^2 +|\hatc_t^{h_t}|^2 \right ) \tag{Cauchy-Schwarz}\\
=& \eta  \rE_{x_t,a_t,r_t|h_t=h}\rE_{f|S_{t-1}} \order\left((\lambda/\eta)^2|\Delta f(x_t^1)|^2+|\BE_t^{h}(f,f;x_t^h,a_t^h)|^2+|\hatc_t^{h}(f))|^2 +|\hatc_t^{h}|^2 \right) \tag{Lemma~\ref{lem:BE-exp}} \\
=& \eta  \rE_{x_t,a_t,r_t|h_t=h}\rE_{f|S_{t-1}} O\left((\lambda/\eta)^2|\Delta f(x_t^1)|+|\BE_t^{h}(f,f;x_t^h,a_t^h)|^2+\hatc_t^{h}(f))\right)
\end{align*}
where the second to the last equation was obtained by taking conditional expectation conditioned on $(x_t^h,a_t^h)$ with respect to the transition and reward, followed by Lemma~\ref{lem:BE-exp}. The last equation used Lemma~\ref{lem:hatc2-1} and Lemma~\ref{lem:hatc2-2}. 
\end{proof}

\begin{lemma}
We have
\[
\lambda \rE \rE_{f|S_{t-1}} \; |\Delta f^1(x_t^1)|
\leq \rE \rE_{f |S_{t-1}} \left[ \lambda \regret(f,x_t^1) + 
0.5 \eta \BE^{h_t}(f,f,x_t^{h_t},a_t^{h_t})^2\right]
+\lambda \tilde{\epsilon}(\lambda/\eta),
\]
 where $\tilde{\epsilon}(\lambda/\eta)$ is defined in Proposition~\ref{lem:decouple}. 
 \label{lem:self-bound-df}
\end{lemma}
\begin{proof}
The proof is similar to that of Proposition~\ref{lem:decouple}. 
\begin{align*}
&\lambda \rE \rE_{f|S_{t-1}} \; |\Delta f^1(x_t^1)|\\
\leq & \lambda \rE \rE_{f_t |S_{t-1}} \regret(f_t,x_t^1)
+ \lambda \sum_{h=1}^H \rE \; \rE_{(x,a)\sim \pi_{f_t}|x_t^h}
     |\BE^{h}(f_t,f_t,x^{h},a^h)| \tag{Proposition~\ref{prop:value-decomposition}}\\
=&   \lambda \rE \rE_{f_t |S_{t-1}} \regret(f_t,x_t^1)
+ \lambda H \rE \; 
    | \BE^{h_t}(f_t,f_t,x_t^{h_t},a_t^{h_t})|
    \tag{$h_t$ is uniformly drawn from $[H]$}\\
      \leq&  \lambda \rE \rE_{f_t |S_{t-1}} \regret(f_t,x_t^1) +
      \lambda\left[ \epsilon_1 H B_1 +  \mu_1^{-1} \br(\epsilon_1) +
          H \mu_1 \rE \;  \rE_{f|S_{t-1}} \BE^{h_t}(f,f,x_t^{h_t},\pi_{f}(x_t^{h_t}))^2\right]
          \tag{Lemma~\ref{lem:decouple-br}}\\
    \leq & \lambda \rE \rE_{f_t |S_{t-1}} \regret(f_t,x_t^1) \\
    &+\lambda\left[2\mu_1 \mu_2 H \rE\;
 \rE_{f|S_{t-1}} \BE^{h_t}(f,f;x_t^{h_t},a_t^{h_t})^{2 } + 
2\mu_1 \mu_2^{-1} \dc(\epsilon_2)
           + 2\epsilon_2 \mu_1 H B_2 + \epsilon_1 H B_1 + \mu_1^{-1} \br(\epsilon_1)\right] . \tag{Lemma~\ref{lem:decouple-dc}}
\end{align*}
   Taking $\mu_1=\mu$ and $\mu_2= \eta/(4\lambda\mu_1 H)$ as in the proof of Proposition~\ref{lem:decouple}, we obtain the desired result. 
\end{proof}
We are now ready to prove Proposition~\ref{lem:inner}.

\paragraph{Proof of Proposition~\ref{lem:inner}}~\\

\begin{proof}
From the proof of Lemma~\ref{lem:hatc}, we obtain
\begin{align*}
&\rE [\rE_{f\sim p_t} \hat{H}_t(f) - \rE_{f\sim p_{t-1}}\hat{H}_{t-1}(f)]\\
=&\rE \; \rE_{f\sim p_{t-1}} 
\left[e^{\lambda\Delta f(x_t^1)-\eta [(\hatd_t^{h_t}(f,f)-\hatc_t^{h_t}(f))-\hatc_t^{h_t}]}-1\right]
 \hat{H}_{t-1}(f)\\
&-\gamma \rE \; \rE_{f\sim p_{t-1}}
\left[e^{\lambda\Delta f(x_t^1)-\eta [(\hatd_t^{h_t}(f,f)-\hatc_t^{h_t}(f))-\hatc_t^{h_t}]}-1\right]
\rE_{g\sim\pp(\cdot|f)}[-\hatd_t^{h_t}(g,f)+\hatc_t^{h_t}(f)] \\
&- \gamma \rE \; \rE_{f\sim p_{t-1}} \rE_{g\sim\pp(\cdot|f)} [-\hatd_t^{h_t}(g,f)+\hatc_t^{h_t}(f)] \\
\leq & \eta c_0 
\rE \; \rE_{f \sim p_{t-1}} O\left(\frac{\lambda}{\eta}|\Delta f(x_t^1)|+|\BE^{h_t}(f,f;x_t^{h_t},a_t^{h_t})|^2+\hatc_t^{h_t}(f)   \right)  (\kappa(\epsilon)+\kappa'(\epsilon)+\ln T) + \frac{\eta}{T} c_0(\kappa(\epsilon)+1)
\tag{Lemma~\ref{lem:hatc-refined-1}}
\\
& + \gamma \eta c_0 \rE \rE_{f \sim p_{t-1}}
O\left(\frac{\lambda}{\eta}|\Delta f(x_t^1)|+|\BE^{h_t}(f,f;x_t^{h_t},a_t^{h_t})|^2+\hatc_t^{h_t}(f)  \right)
\tag{Lemma~\ref{lem:hatc-refined-2}}\\
& - \gamma \rE \; \rE_{f\sim p_{t-1}} \hatc_t^{h_t}(f)
+ \order(\epsilon) \tag{$\forall g \in \cF(\epsilon,f)$, we have $\hatd_t^{h_t}(g,f)=\order(\epsilon)$}\\
\leq & c_1 \gamma (\lambda/\eta) \rE \; \rE_{f \sim p_{t-1}}
O(|\Delta f(x_t^1)|) + 
c_1 \gamma  
\rE \; \rE_{f \sim p_{t-1}} O(|\BE^{h_t}(f,f;x_t^{h_t},a_t^{h_t})|^2)
 - 0.5 \gamma \rE \; \rE_{f\sim p_{t-1}} \hatc_t^{h_t}(f)\\
&+ c_0 O(\epsilon + \eta(\kappa(\epsilon)+1)/T) .
\end{align*}
The last inequality used the condition of the Lemma concerning the choice of $\eta$. Now by summing over $t=1$ to $t=T$, and apply Lemma~\ref{lem:hatH} to bound $\rE \hat{H}_{0}(f)$, we obtain when $c_0$ is sufficiently large

\begin{align*}
  &\gamma \sum_{t=1}^T \rE \; \rE_{f\sim p_{t-1}} \hatc_t^{h_t}(f) \\
  \leq &    c_0\gamma O(\epsilon T + \kappa(\epsilon)+\kappa'(\epsilon)+1) + 2 c_1 \gamma\sum_{t=1}^T \rE \; \rE_{f\sim p_{t-1}} O\left[|\BE^{h_t}(f,f;x_t^{h_t},a_t^{h_t})|^2 + (\lambda/\eta) |f(x_t^1)|\right]\\
  \leq &  c_0\gamma O(\epsilon T + \kappa(\epsilon)+\kappa'(\epsilon)+1) 
  + 2 c_1 \gamma\sum_{t=1}^T \rE \; \rE_{f\sim p_{t-1}} O\left[|\BE^{h_t}(f,f;x_t^{h_t},a_t^{h_t})|^2 \right]\\
    & + 2  c_1 \gamma \sum_{t=1}^T
  \rE \rE_{f |S_{t-1}} O\left[ (\lambda/\eta) \regret(f,x_t^1) + 
0.5  \BE^{h_t}(f,f,x_t^{h_t},a_t^{h_t})^2\right]
+2 c_1 \gamma (\lambda/\eta) \tilde{\epsilon}(\lambda/\eta) O(T) .
\end{align*}

The proof is now completed by recalling Lemma~\ref{lem:hatdc0} with $\alpha'<6$.
\end{proof}

\section{Proof of Theorem~\ref{thm:main-des}}
\label{sec:proof-des}

In this section, we provide a proof for Theorem~\ref{thm:main-des}. We focus on the differences from the proof of Theorem~\ref{thm:main}, which are limited to our decoupling analysis. In particular, in the setting of Theorem~\ref{thm:main-des}, we have the following improved analog for Lemma~\ref{lem:decouple-dc}.

\begin{lemma}[Design based decoupling]
We have for all $x^h \in \cX^h$ and $f\in\cF$:
\begin{align*}
\BE^h(f,f;x^h,\pi_{f}(x^h))^2
\leq \BE^h(f,f,x^h,\rho^h(x^h))^2,
\end{align*}
where $\rho^h(x^h)$ is the $G$-optimal design~\eqref{eq:g-opt} at the state $x^h$.
\label{lem:decouple-des}
\end{lemma}

\begin{proof}
Let $\Sigma_{\star}(x^h)$ be the covariance $\rE_{a^h\sim \rho(x^h)} [\phi^h(x^h, a^h)\phi^h(x^h,a^h)^\top]$. The definition of $G$-optimal design implies that
\begin{equation}
\sup_{a^h \in \cA}
\|\phi^h(x^h,a^h)\|_{\Sigma_{\star}(x^h)^{-1}}^2
\leq d_2 . 
\label{eq:proof-decouple-des-G}
\end{equation}

By Assumption~\ref{assumption:embedding}, we have 

\begin{align*}
\BE^h(f,f;x^h,\pi_{f}(x^h))^2 &= \inner{w^h(f,x^h)}{\phi^h(x^h,a)}^2\\
&\leq \|w^h(f,x^h)\|_{\Sigma_{\star}(x^h)}^2\|\phi^h(x^h,a^h)\|_{\Sigma_{\star}(x^h)^{-1}}^2\\
&\leq d_2\rE_{a^h\sim \rho^h(x^h)} \inner{w^h(f,x^h)}{\phi^h(x^h,a^h)}^2\\
&= d_2\BE^h(f,f,x^h,\rho^h(x^h))^2.
\end{align*}
The first inequality used Cauchy-Schwarz. The second inquality used \eqref{eq:proof-decouple-des-G}. 
\end{proof}

Combining this with the Bellman rank decoupling in Lemma~\ref{lem:decouple-br} and Proposition~\ref{prop:value-decomposition}, we get

\begin{align*}
   &\lambda \rE \; \regret(f_t,x_t^1)+\lambda \rE \rE_{f|S_{t-1}} \;\Delta f^1(x_t^1)
     = \lambda H\rE \; 
     \BE^{h_t}(f_t,f_t,x_t^{h_t},a_t^{h_t})\\
    \leq&   \lambda\left[ \epsilon_1 H B_1 +  \mu_1^{-1} \br(\epsilon_1) +
          H \mu_1 \rE \;  \rE_{f|S_{t-1}} \BE^{h_t}(f,f,x_t^{h_t},\pi_{f}(x_t^{h_t}))^2\right]\\
    \leq&\lambda\left[ \epsilon_1 H B_1 +  \mu_1^{-1} \br(\epsilon_1) +
          H \mu_1 d_2\rE \;  \rE_{f|S_{t-1}} \BE^{h_t}(f,f,x_t^{h_t},\rho^{h_t}(x^{h_t}_t))^2\right]\\
    =&\lambda\left[ \epsilon_1 H B_1 +  \mu_1^{-1} \br(\epsilon_1) +
          H \mu_1 d_2\rE \;  \rE_{f|S_{t-1}} \BE^{h_t}(f,f,x_t^{h_t},a^{h_t}_t)^2\right]\\
\end{align*}

Further choosing $\lambda H\mu_1 d_2 = 0.5\eta$, we get the following analog of Proposition~\ref{lem:decouple}.

\begin{align*}
    \lambda \rE \; \regret(f_t,x_t^1)\leq & \rE \; \rE_{f|S_{t-1}} \left[-\lambda \Delta f(x_t^1)
    + 0.5 \eta \BE^{h_t}(f,f,x_t^{h_t},a_t^{h_t})^2\right]\\
    &+\lambda \underbrace{\left(\epsilon_1 H B_1 + \frac{2\lambda \br(\epsilon_1) H_1d_2}{\eta}\right)}_{\tilde{\epsilon}'(\lambda/\eta)}.
\end{align*}

Now we combine this result with 
Proposition~\ref{lem:regret-hatc-Z} 
(the result still hold without modification, because the proofs did not rely on how $a_t^h$ is generated given $x_t^h$), and a slight modification
of 
Proposition~\ref{lem:inner} with $\tilde{\epsilon}(\lambda/\eta)$ replaced by 
$\tilde{\epsilon}'(\lambda/\eta)$ (we just need to replace the decoupling used in Lemma~\ref{lem:self-bound-df} by the new decoupling result used here), leading to 
\begin{align*}
    \sum_{t=1}^T \lambda \rE \; \regret(f_t,x_t^1)\leq &  
    \order\left(\lambda \epsilon T +4
 \eta T \epsilon^2 +  \kappa(\epsilon) + 1.5 \lambda^2 T\right)\\
 &+\order(\epsilon T + \kappa(\epsilon)+\kappa'(\epsilon)+1)\\
    &+\order\left(\lambda T  \tilde{\epsilon}'(\lambda/\eta) \right) .
\end{align*}
This implies 
the result of Theorem~\ref{thm:main-des}.

\end{document}